\documentclass[11pt]{article}

\usepackage{subfig}
\usepackage{tikz}
\usepackage{dsfont}
\usepackage{pgf,pgfplots}
\usepackage{mathrsfs}
\def\colorful{1}

\usepackage{amsthm,amsfonts,amsmath,amssymb,epsfig,color,float,graphicx,verbatim, enumitem}
\usepackage{multirow}
\usepackage{algorithm}
\usepackage[noend]{algpseudocode}

\newif\ifhyper\IfFileExists{hyperref.sty}{\hypertrue}{\hyperfalse}
\hypertrue
\ifhyper\usepackage{hyperref}\fi

\usepackage{enumitem}

\usepackage{environ}

\newcommand{\repeattheorem}[1]{\begingroup
  \renewcommand{\thetheorem}{\ref{#1}}\expandafter\expandafter\expandafter\theorem
  \csname reptheorem@#1\endcsname
  \endtheorem
  \endgroup
}

\newcommand{\repeatlemma}[1]{\begingroup
  \renewcommand{\thelemma}{\ref{#1}}\expandafter\expandafter\expandafter\lemma
  \csname replemma@#1\endcsname
  \endlemma
  \endgroup
}

\NewEnviron{replemma}[1]{\global\expandafter\xdef\csname replemma@#1\endcsname{\unexpanded\expandafter{\BODY}}\expandafter\lemma\BODY\unskip\label{#1}\endlemma
}

\NewEnviron{reptheorem}[1]{\global\expandafter\xdef\csname reptheorem@#1\endcsname{\unexpanded\expandafter{\BODY}}\expandafter\theorem\BODY\unskip\label{#1}\endtheorem
}

\def\nnewcolor{1}
\ifnum\nnewcolor=1

\fi
\ifnum\nnewcolor=0

\fi

\ifnum\colorful=1

\else

\fi

\newtheorem{theorem}{Theorem}[section]

\newtheorem{lemma}[theorem]{Lemma}
\newtheorem{informal theorem}[theorem]{Theorem (informal statement)}

\newtheorem{corollary}[theorem]{Corollary}
\newtheorem{claim}[theorem]{Claim}
\newtheorem{fact}[theorem]{Fact}

\theoremstyle{definition}
\newtheorem{definition}[theorem]{Definition}
\newcommand{\eqdef}{\stackrel{{\mathrm {\footnotesize def}}}{=}}

\newtheorem{innercustomgeneric}{\customgenericname}
\providecommand{\customgenericname}{}
\newcommand{\newcustomtheorem}[2]{\newenvironment{#1}[1]
  {\renewcommand\customgenericname{#2}\renewcommand\theinnercustomgeneric{##1}\innercustomgeneric
  }
  {\endinnercustomgeneric}
}

\newcustomtheorem{customlem}{Lemma}
\newcustomtheorem{customclm}{Claim}

\newcommand{\sigMO}[1]{ e^{-\frac{#1}{\sigma }}}
\newcommand{\density}{\gamma}
\newcommand{\W}{ {\cal W} }
\newcommand{\ramp}{r_{\sigma}}
\newcommand{\lp}{\left}
\newcommand{\rp}{\right}
\newcommand\norm[1]{\left\| #1 \right\|}
\newcommand\snorm[2]{\left\| #2 \right\|_{#1}}
\renewcommand\vec[1]{\mathbf{#1}}
\DeclareMathOperator*{\Prob}{\mathbf{Pr}}
\DeclareMathOperator*{\E}{\mathbf{E}}
\newcommand{\proj}{\mathrm{proj}}
\newcommand{\me}{\mathrm{e}}
\newcommand{\mrm}{\mathrm}
\newcommand{\SL}{\mathcal{L}_{\sigma}}
\newcommand{\SLS}{\mathcal{L}_{\sigma}^{\text{ramp}}}

\def\d{\mathrm{d}}

\DeclareMathOperator*{\argmin}{argmin}

\newcommand{\bx}{\mathbf{x}}
\newcommand{\by}{\mathbf{y}}

\newcommand{\bw}{\mathbf{w}}

\newcommand{\Sp}{\mathbb{S}}

\newcommand{\err}{\mathrm{err}}

\newcommand{\R}{\mathbb{R}}

\newcommand{\Z}{\mathbb{Z}}

\newcommand{\eps}{\epsilon}

\newcommand{\pr}{\mathbf{Pr}}
\newcommand{\poly}{\mathrm{poly}}

\newcommand{\sgn}{\mathrm{sign}}
\newcommand{\sign}{\mathrm{sign}}

\newcommand{\opt}{\mathrm{OPT}}
\newcommand{\D}{\mathcal{D}}

\newcommand{\Ind}{\mathds{1}}
\newcommand{\1}{\Ind}

\newcommand{\littlesum}{\mathop{\textstyle \sum}}

\newcommand{\wh}{\widehat}
\usetikzlibrary{angles, quotes}
\usepackage{pgfplots}

\newcommand{\dotp}[2]{\left\langle #1, #2 \right\rangle}

\newcommand{\wstar}{\bw^{\ast}}
\newcommand{\citep}{\cite}

\title{Learning Halfspaces with Massart Noise \\ Under Structured Distributions}

\author{
Ilias Diakonikolas\thanks{Supported by NSF Award CCF-1652862 (CAREER) and a Sloan Research Fellowship.}\\
University of Wisconsin-Madison\\
{\tt ilias@cs.wisc.edu}\\
\and
Vasileios Kontonis\\
University of Wisconsin-Madison\\
{\tt kontonis@wisc.edu }\\
\and
Christos Tzamos\\ University of Wisconsin-Madison\\
{\tt tzamos@wisc.edu}
\and
Nikos Zarifis\\
University of Wisconsin-Madison\\
{\tt zarifis@wisc.edu}\\
}

\begin{document}

\maketitle

\begin{abstract}
We study the problem of learning halfspaces with Massart noise in the distribution-specific PAC model.
We give the first computationally efficient algorithm for this problem
with respect to a broad family of distributions,
including log-concave distributions. This resolves an open question posed in a number of prior works.
Our approach is extremely simple: We identify a smooth {\em non-convex} surrogate loss with the
property that any approximate stationary point of this loss defines a halfspace
that is close to the target halfspace. Given this structural result, we can use SGD
to solve the underlying learning problem.
\end{abstract}

\setcounter{page}{0}
\thispagestyle{empty}
\newpage

\section{Introduction} \label{sec:intro}

\subsection{Background and Motivation} \label{ssec:background}

Halfspaces, or Linear Threshold Functions, are Boolean functions $h_{\bw}: \R^d \to \{ \pm 1\}$ of the form
$h_{\bw}(\bx) = \sgn \left(\langle \bw, \bx \rangle \right)$, where $\bw \in \R^d$ is the associated weight vector.
(The univariate function $\sign(t)$ is defined as $\sgn(t)=1$, for $t \geq 0$, and $\sgn(t)=-1$ otherwise.)
Halfspaces have been a central object of study in various fields, including complexity theory, optimization,
and machine learning~\citep{MinskyPapert:68, Yao:90, GHR:92, CristianiniShaweTaylor:00, AoBF14}.
Despite being studied over several decades, a number of basic structural and algorithmic questions involving halfspaces
remain open.

The algorithmic problem of learning an unknown halfspace from random labeled examples
has been extensively investigated since the 1950s --- starting with Rosenblatt's Perceptron
algorithm~\citep{Rosenblatt:58} --- and has arguably been one of
the most influential problems in the field of machine learning.
In the realizable case, i.e., when all the labels are consistent with the target halfspace,
this learning problem amounts to linear programming, hence can be solved in polynomial time
(see, e.g.,~\cite{MT:94, CristianiniShaweTaylor:00}).
The problem turns out to be much more challenging algorithmically
in the presence of noisy labels, and its computational complexity crucially depends on the noise model.

In this work, we study the problem of distribution-specific PAC
learning of halfspaces in the presence of Massart noise~\citep{Massart2006}.
In the Massart noise model, an adversary can flip each label independently
with probability {\em at most} $\eta<1/2$, and the goal of the learner is to reconstruct
the target halfspace to arbitrarily high accuracy. More formally, we have:

\begin{definition}[Distribution-specific PAC Learning with Massart Noise] \label{def:massart-learning}
Let $\mathcal{C}$ be a concept class of Boolean functions over $X= \R^d$,
$\mathcal{F}$ be a {\em known family} of structured distributions on $X$,  $0 \leq \eta <1/2$,
and $0< \eps <1$.
Let $f$ be an unknown target function in $\mathcal{C}$.
A {\em noisy example oracle}, $\mathrm{EX}^{\mathrm{Mas}}(f, \mathcal{F}, \eta)$,
works as follows: Each time $\mathrm{EX}^{\mathrm{Mas}}(f, \mathcal{F}, \eta)$ is invoked,
it returns a labeled example $(\bx, y)$, such that: (a) $\bx \sim \D_{\bx}$, where $\D_{\bx}$ is a fixed
distribution in $\mathcal{F}$, and (b) $y = f(\bx)$ with probability $1-\eta(\bx)$
and $y = -f(\bx)$ with probability $\eta(\bx)$, for an {\em unknown} parameter $\eta(\bx) \leq \eta$.
Let $\D$ denote the joint distribution on $(\bx, y)$ generated by the above oracle.
A learning algorithm is given i.i.d. samples from $\D$
and its goal is to output a hypothesis $h$ such that with high probability $h$ is $\eps$-close to
$f$, i.e., it holds $\pr_{\bx \sim \D_{\bx}} [h(\bx) \neq f(\bx)] \leq \eps$.
\end{definition}

Massart noise is a realistic model of random noise that has attracted significant attention
in recent years (see Section~\ref{ssec:related} for a summary of prior work).
This noise model goes back to the 80s, when it was
studied by Rivest and Sloan~\citep{Sloan88, RivestSloan:94}
under the name ``malicious misclassification noise'', and a very similar asymmetric
noise model was considered even earlier by Vapnik~\citep{Vapnik82}.
The Massart noise condition lies in between the
Random Classification Noise (RCN)~\citep{AL88} -- where each label is flipped
independently with probability {\em exactly} $\eta<1/2$ --
and the agnostic model~\citep{Haussler:92, KSS:94} -- where an adversary can flip
any small constant fraction of the labels.

The sample complexity of PAC learning with Massart noise is well-understood.
Specifically, if $\mathcal{C}$ is the class of $d$-dimensional halfspaces,
it is well-known~\citep{Massart2006} that $O(d/(\eps \cdot (1-2\eta)^2))$
samples information-theoretically suffice to determine a hypothesis $h$
that is $\eps$-close to the target halfspace $f$ with high probability
(and this sample upper bound is best possible).
The question is whether a computationally efficient algorithm exists.

The algorithmic question of efficiently computing an accurate hypothesis in the distribution-specific PAC setting
with Massart noise was initiated in~\cite{AwasthiBHU15}, and subsequently studied in a sequence of
works~\citep{AwasthiBHZ16, ZhangLC17, YanZ17, MangoubiV19}. This line of work has given
polynomial-time algorithms for learning halfspaces with Massart
noise, when the underlying marginal distribution $\D_{\bx}$ is the uniform distribution
on the unit sphere (i.e., the family $\mathcal{F}$ in Definition~\ref{def:massart-learning} is a singleton).

The question of designing a computationally efficient learning algorithm for this problem that succeeds
under more general distributional assumptions remained open, and has been posed as an open problem
in a number of places~\citep{AwasthiBHZ16, Awasthi:18-ttic, BH20}.
Specifically,~\cite{AwasthiBHZ16} asked whether there exists a polynomial-time
algorithm for all log-concave distributions, and the same question was more recently highlighted in~\cite{BH20}.
In more detail,~\cite{AwasthiBHZ16} gave an algorithm that succeeds under any log-concave distribution,
but has sample complexity and running time $d^{2^{\poly(1/(1-2\eta))}}/\poly(\eps)$, i.e., doubly exponential
in $1/(1-2\eta)$.
\cite{BH20} asked whether a $\poly(d, 1/\eps, 1/(1-2\eta))$ time algorithm exists for log-concave marginals.
As a corollary of our main algorithmic result (Theorem~\ref{thm:main-inf}), we answer this question in the affirmative. Perhaps surprisingly, our algorithm is extremely simple
(performing SGD on a natural non-convex surrogate) and succeeds
for a broader family of structured distributions,
satisfying certain (anti)-anti-concentration and tail bound properties.
In the following subsection, we describe our main contributions in detail.

\subsection{Our Results} \label{ssec:results}

The main result of this paper is the first polynomial-time algorithm for learning
halfspaces with Massart noise with respect to a broad class of well-behaved distributions.
Before we formally state our algorithmic result, we define the
family of distributions $\mathcal{F}$ for which our algorithm succeeds:

\begin{definition}[Bounded distributions] \label{def:bounds}
Fix $U, R >0$ and $t: (0,1) \rightarrow \R_+$.
An isotropic (i.e., zero mean and identity covariance) distribution $\D_{\bx}$ on $\R^d$
is called $(U, R, t)$-bounded if for any projection $(\D_{\bx})_V$ of $\D_{\bx}$
onto a $2$-dimensional subspace $V$ the corresponding pdf $\gamma_V$ on $\R^2$
satisfies the following properties:
\begin{enumerate}
  \item  $\gamma_V(\vec x) \geq 1/U$, for all $\vec x \in V$ such that $\snorm{2}{\vec x} \leq R$ (anti-anti-concentration).
  \item $\gamma_V(\vec x) \leq U$ for all $x \in V$ (anti-concentration).
  \item For any $\eps \in (0,1)$, $\pr_{\vec x \sim \gamma_V}[ \snorm{2}{\vec x} \geq t(\eps)] \leq \eps$
  (concentration).
\end{enumerate}
We say that $\D_{\bx}$ is $(U, R)$-bounded if concentration is not required to hold.
\end{definition}

The main result of this paper is the following:

\begin{theorem}[Learning Halfspaces with Massart Noise] \label{thm:main-inf}
There is a computationally efficient algorithm that learns halfspaces in the presence of Massart noise
with respect to the class of $(U, R, t)$-bounded distributions on $\R^d$. Specifically, the algorithm
draws $m = \poly\left(U/R, t(\eps/2), 1/(1-2\eta) \right) \cdot O(d/\eps^4)$ samples from a noisy example
oracle at rate $\eta<1/2$, runs in sample-polynomial time, and outputs a hypothesis halfspace $h$ that is $\eps$-close to the target with probability at least $9/10$.
\end{theorem}

See Theorem~\ref{thm:main_massart} for a more detailed statement.
Theorem~\ref{thm:main-inf} provides the first polynomial-time algorithm
for learning halfspaces with Massart noise under a fairly broad family of well-behaved
distributions. Specifically, our algorithm runs in $\poly(d, 1/\eps, 1/(1-2\eta))$ time, as long
as the parameters $R, U$ are bounded above by some $\poly(d)$,
and the function $t(\eps)$ is bounded above by some $\poly(d/\eps)$.
These conditions do not require a specific parametric or nonparametric form for the underlying density
and are satisfied for several reasonable continuous distribution families.
We view this as a conceptual contribution of this work.

It is not hard to show that the class of isotropic log-concave distributions
is $(U,R, t)$-bounded, for $U, R = O(1)$ and $t(\eps) = O(\log(1/\eps))$
(see Fact~\ref{fact:RU-lc}). Similar implications hold for a broader class of distributions, known as
$s$-concave distributions. (See Appendix~\ref{app:lc-sc}.)
Using Fact~\ref{fact:RU-lc},
we immediately obtain the following corollary:

\begin{corollary}[Learning Halfspaces with Massart Noise Under Log-concave Distributions]  \label{cor:lc}
There exists a polynomial-time algorithm that learns halfspaces with Massart noise under any isotropic log-concave distribution.
The algorithm has sample complexity $m = \widetilde{O}(d/\eps^4) \cdot \poly(1/(1-2\eta))$ and runs in sample-polynomial time.
\end{corollary}

Corollary~\ref{cor:lc} gives the first polynomial-time algorithm for this problem, answering
an open question of~\cite{AwasthiBHZ16, Awasthi:18-ttic, BH20}.
We obtain similar implications for $s$-concave distributions.
(See Appendix~\ref{app:lc-sc} for more details.)

While the preceding discussion focused on polynomial learnability,
our algorithm establishing Theorem~\ref{thm:main-inf}
is extremely simple and can potentially be practical. Specifically, our algorithm
simply performs SGD (with projection on the unit ball) on a natural
{\em non-convex} surrogate loss, namely an appropriately smoothed
version of the misclassification error function,
$\err_{0-1}^{\mathcal{D}}(\bw) = \pr_{(\bx,y) \sim \D} [\sign(\langle \bx, \bw\rangle) \neq y]$.
We also note that the sample complexity of our algorithm for log-concave marginals
is optimal as a function of the dimension $d$, within constant factors.

Our approach for establishing Theorem~\ref{thm:main-inf} is fairly robust and
immediately extends to a slightly stronger noise model, considered in~\cite{ZhangLC17},
which we term {\em strong Massart noise}. In this model, the flipping probability can
be arbitrarily close to $1/2$ for points that are very close to the true separating hyperplane.
These implications are stated and proved in Section~\ref{sec:generalized}.

\subsection{Technical Overview} \label{ssec:techniques}

Our approach is extremely simple: We take an optimization view and leverage the
structure of the learning problem to identify a simple {\em non-convex}
surrogate loss $\SL(\bw)$ with the following property: {\em Any} approximate
stationary point $\wh{\bw}$ of $\SL$ defines a halfspace $h_{\wh{\bw}}$, which
is close to the target halfspace $f(\bx) = \sign(\langle \bw^{\ast}, \bx
\rangle)$.  Our non-convex surrogate is smooth, by design. Therefore, we can
use any first-order method to efficiently find an approximate stationary point.

We now proceed with a high-level intuitive explanation. For simplicity of this
discussion, we consider the population versions of the relevant loss functions.
The most obvious way to solve the learning problem is by attempting to directly
optimize the population risk with respect to the $0-1$ loss, i.e., the
misclassification error $\pr_{(\bx,y) \sim \D} [h_{\bw}(\bx) \neq y]$ as a
function of the weight vector $\bw$.  Equivalently, we seek to minimize the
function $F(\bw) = \E_{(\bx,y) \sim \D} [\1\{- y \dotp{\bw}{\bx} \geq 0\}]$, where
$\1\{t \geq 0\}$ is the zero-one step function.  This is of course a non-convex
problem and it is unclear how to efficiently solve directly.

A standard recipe in machine learning to address non-convexity is to replace
the $0-1$ loss $F(\bw)$ by an appropriate convex surrogate. This method seems
to inherently fail in our setting.  However, we are able to find a {\em
non-convex} surrogate that works. Even though finding a global optimum of a
non-convex function is hard in general, we show that a much weaker requirement
suffices for our learning problem. In particular, it suffices to find a point
where our non-convex surrogate has small gradient. Our main structural result
is that any such point is close to the target weight vector $\bw^{\ast}$.

To obtain our non-convex surrogate loss $\SL$, we replace the step function
$\1\{t \geq 0\}$ in $F(\bw)$ by a well-behaved approximation. That is, our
surrogate is of the form 
$\SL(\bw) = \E_{(\bx,y) \sim \D} [ r(-y \dotp{\bw}{\bx})]$, 
where $r(t)$ is an approximation (in some sense) of
$\1\{t \geq 0\}$. A natural first idea is to approximate the step function by a
piecewise linear (ramp) function. We show (Section~\ref{ssec:ramp}) that this
leads to a non-convex surrogate that indeed satisfies the desired structural
property. The proof of this statement turns out to be quite clean, capturing
the key intuition of our approach. Unfortunately, the non-convex surrogate
obtained this way (i.e., using the ramp function as an approximation to the
step function) is non-smooth and it is unclear how to efficiently find an
approximate stationary point.  A simple way to overcome this obstacle is to
instead use an appropriately {\em smooth} approximation to the step function.
Specifically, we use the logistic loss (Section~\ref{ssec:smooth-sigmoid}), but
several other choices would work.  See Figure~\ref{fig:lossfunctions} for an illustration.

\begin{figure}
  \centering
  \includegraphics[width=0.4\textwidth]{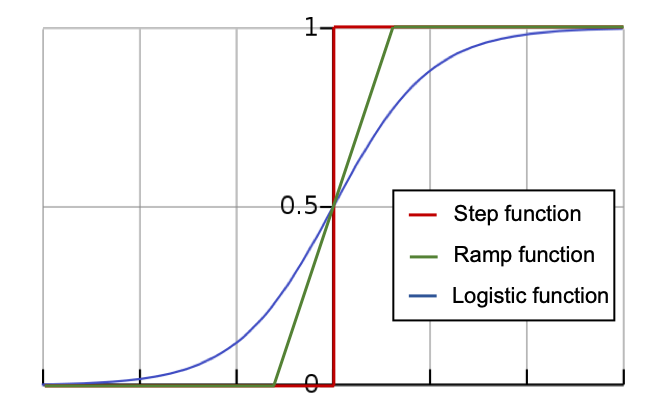}
  \caption{The step function and its surrogates.}
\label{fig:lossfunctions}
\end{figure}

We note that our structural lemma (showing that any stationary point of a non-convex surrogate suffices) 
crucially leverages the underlying distributional assumptions (i.e., the fact that $\D_{\bx}$ is $(U, R)$
bounded). It follows from a lower bound construction in~\cite{DGT19} that the
approach of the current paper does not extend to the distribution-independent
setting.  In particular, for any loss function $\mathcal{L}$, \cite{DGT19} constructs
examples where there exist stationary points of $\mathcal{L}$ defining
hypotheses that are far from the target halfspace.

\subsection{Related and Prior Work} \label{ssec:related}

\paragraph{Prior Work on Learning with Massart Noise}
We start with a summary of prior work on distribution-specific PAC learning of
halfspaces with Massart noise. The study of this learning problem was initiated
in~\cite{AwasthiBHU15}. That work gave the first polynomial-time algorithm for
the problem that succeeds under the uniform distribution on the unit sphere,
assuming the upper bound on the noise rate $\eta$ is smaller than a
sufficiently small constant ($\approx 10^{-6}$). Subsequently,
\cite{AwasthiBHZ16} gave a learning algorithm with sample and computational
complexity $d^{2^{\poly(1/(1-2\eta))}}/\poly(\eps)$ that succeeds for any noise
rate $\eta<1/2$ under any log-concave distribution.

The approach in~\cite{AwasthiBHU15, AwasthiBHZ16} uses an iterative
localization-based method. These algorithms operate in a sequence of phases and
it is shown that they make progress in each phase.
To achieve this,~\cite{AwasthiBHU15, AwasthiBHZ16} leverage a
distribution-specific agnostic learner for halfspaces~\citep{KKMS:08} and
develop sophisticated tools to control the trajectory of their algorithm.

Inspired by the localization approach, \cite{YanZ17} gave a 
perceptron-like algorithm (with sample complexity linear in $d$) 
for learning halfspaces with Massart noise under the uniform distribution on the sphere.
Their algorithm again proceeds in phases and crucially exploits the symmetry
of the uniform distribution to show that the angle between the current
hypothesis $\wh{\bw}^{(i)}$ and the target halfspace $\wstar$ decreases in
every phase.  \cite{ZhangLC17} also gave a polynomial-time algorithm for
learning halfspaces with Massart noise under the uniform distribution on the
unit sphere.  Their algorithm works in the strong Massart noise model and is
based on the Stochastic Gradient Langevin Dynamics (SGLD) algorithm applied to a
smoothed version of the empirical $0-1$ loss. Their method leads to sample
complexity $\Omega_{\eta}(d^4/\eps^4)$ and its running time involves
$\Omega_{\eta}(d^{13.5}/\eps^{16})$ inner product evaluations.  More recently,
\cite{MangoubiV19} improved these bounds to
$\Omega_{\eta}(d^{8.2}/\eps^{11.4})$ inner product evaluations via a similar
approach. Our method is much simpler in comparison, running SGD
directly on the population loss and using one sample per iteration with a
significantly improved sample complexity and running time.

Furthermore, in contrast to the aforementioned approaches, we study a more general setting
(in the sense that our method works for a broad family of distributions), and
our approach is not tied to the iterations of any particular algorithm. Our
structural lemma (Lemma \ref{lem:structural_massart}) shows that {\em any}
approximate stationary point of our non-convex surrogate loss suffices. As a
consequence, one can apply any first-order method that converges to
stationarity (and in particular vanilla SGD with projection on the unit sphere
works). The upshot is that we do not need to establish guarantees for the
trajectory of the method used to reach such a stationary point.  The only thing
that matters is the endpoint of the algorithm. Intriguingly, for a generic
distribution in the class we consider, it is unclear if it is possible to
establish a monotonicity property for a first-order method reaching a
stationary point.

We note that the $d$-dependence in the sample complexity of our algorithm
is information-theoretically optimal, even under Gaussian marginals.
The $\eps$-dependence seems tight for our approach,  given recent lower bounds
for the convergence of SGD~\citep{Dr19lb}, or any stochastic first-order method~\citep{Arj19lb},
to stationary points of smooth non-convex functions.

Finally, we comment on the relation to a recent work
on distribution-independent PAC learning of halfspaces with Massart noise~\citep{DGT19}.
\cite{DGT19} gave a distribution-independent PAC learner
for halfspaces with Massart noise that approximates
the target halfspace within misclassification error $\approx \eta$, i.e.,
it does not yield an arbitrarily close approximation to the true function.
In contrast, the aforementioned distribution-specific algorithms
achieve information-theoretically optimal misclassification error, which implies
that the output hypothesis can be arbitrarily close to the true target halfspace.
As a result, the results of this paper are not subsumed by~\cite{DGT19}.

\paragraph{Comparison to RCN and Agnostic Settings}
It is instructive to compare the complexity of learning halfspaces in the
Massart model with two related noise models. In the RCN model,
a polynomial-time algorithm is known in the distribution-independent
PAC model~\citep{BlumFKV96, BFK+:97}.
In sharp contrast, even weak agnostic learning is hard in the
distribution-independent setting~\citep{GR:06, FGK+:06short, Daniely16}.
Moreover, obtaining information-theoretically optimal error guarantees
remains computationally hard in the agnostic model, even when the marginal
distribution is the standard Gaussian~\citep{KlivansK14} (assuming the hardness of noisy parity).
On the other hand, recent work~\citep{ABL17, DKS18a} has given efficient algorithms
(for Gaussian and log-concave marginals)
with error $O(\opt) +\eps$, where $\opt$ is the misclassification error of the optimal halfspace.

 \newcommand{\capfun}{\mathrm{cap}}
\newcommand{\CLR}{\mathrm{CappedLeakyRelu}}

\section{Preliminaries}

For $n \in \Z_+$, let $[n] \eqdef \{1, \ldots, n\}$.  We will use small
boldface characters for vectors.  For $\bx \in \R^d$ and $i \in [d]$, $\bx_i$
denotes the $i$-th coordinate of $\bx$, and $\|\bx\|_2 \eqdef
(\littlesum_{i=1}^d \bx_i^2)^{1/2}$ denotes the $\ell_2$-norm of $\bx$. 
We will use $\langle \bx, \by \rangle$ for the inner product of $\bx, \by \in
\R^d$ and $ \theta(\bx, \by)$ for the angle between $\bx, \by$.

Let $\vec e_i$ be the $i$-th standard basis vector in $\R^d$.  
For $d\in \mathbb{N}$, let $\Sp^{d-1} \eqdef \{\bx \in \R^d:\|\bx\|_2 = 1 \}$. 
Let $\proj_U(\vec x)$ be the projection of $\vec x$ to subspace
$U \subset \R^d$ and $U^{\perp}$ be its orthogonal complement.  

Let $\E[X]$ denote the expectation of random variable $X$ and
$\pr[\mathcal{E}]$ the probability of event $\mathcal{E}$.  

An (origin-centered) halfspace is any Boolean-valued function $h_{\bw}: \R^d \to \{\pm 1\}$ 
of the form $h_{\bw}(\bx) = \sgn \left(\langle \bw, \bx \rangle \right)$,
where $\bw \in \R^d$. (Note that we may assume w.l.o.g. that $\|\bw\|_2 =1$.)

We consider the binary classification setting where labeled examples $(\bx,y)$ are drawn
i.i.d. from a distribution $\D$ on $\R^d \times \{ \pm 1\}$. 
We denote by $\D_{\bx}$ the marginal of $\D$ on $\vec x$.
The misclassification error of a hypothesis $h: \R^d \to \{\pm 1\}$ (with respect to $\D$) is 
$\err_{0-1}^{\D}(h) \eqdef \pr_{(\bx, y) \sim \D}[h(\bx) \neq y]$. The zero-one error between
two functions $f, h$ (with respect to $\D_{\bx}$) is 
$\err_{0-1}^{\D_{\bx}}(f, h) \eqdef \pr_{\bx \sim \D_{\bx}}[f(\bx) \neq h(\bx)]$.

We will use the following simple claim relating the zero-one loss between two halfspaces (with respect
to a bounded distribution) and the angle between their normal vectors (see Appendix~\ref{app:angle-loss} 
for the proof).
\begin{claim}\label{lem:angle_zero_one}
Let $\D_{\bx}$ be a $(U, R)$-bounded distribution on $\R^d$. 
For any $\vec u, \vec v \in \R^d$ we have that 
$R^2/ U \cdot \theta(\vec u, \vec v) \leq \err_{0-1}^{\D_{\bx}}(h_{\vec u},h_{\vec v})$.
Moreover, if $\D_{\bx}$ is $(U, R, t(\cdot))$-bounded, for any $0< \eps \leq 1$,\ we have that 
$\err_{0-1}^{\D_{\bx}}(h_{\vec u},h_{\vec v})
\leq U t(\eps)^2 \cdot \theta(\vec v, \vec u) + \eps \;.$
\end{claim}

 \section{Main Structural Result: Stationary Points Suffice}\label{section3}\label{sec:struct}
In this section, we prove our main structural result.  In
Section~\ref{ssec:ramp}, we define a simple non-convex surrogate by replacing
the step function by the (piecewise linear) ramp function and show that any
approximate stationary point of this surrogate loss suffices. In
Section~\ref{ssec:smooth-sigmoid}, we prove our actual structural result for
a smooth (sigmoid-based) approximation to the step function.

\subsection{Warm-up: Non-convex surrogate based on ramp function} \label{ssec:ramp}
The main point of this subsection is to illustrate the key properties of a
non-convex surrogate loss that allows us to argue that the stationary points
of this loss are close to the true halfspace $\vec w^{\ast}$.  To this end,
we consider the \emph{ramp function} $\ramp(t)$ with parameter $\sigma>0$ --
a piecewise linear approximation to the step function. The ramp function and
its derivative are defined as follows:
\begin{align}
  \label{ramp:eq:ramp_definition}
  \ramp(t) =
  \begin{cases}
    0,  &\text{for }t< -\sigma/2  \\
    \frac{t}{\sigma} +  \frac{1}{2}, & |t| \leq \sigma/2 \\
    1,  &t > \sigma/2 \\
  \end{cases}
  \qquad \text{and} \qquad
  \ramp'(t) =
  \frac{1}{\sigma} \1\{|t| \leq \sigma/2\}\;.
\end{align}
Observe that as $\sigma$ approaches $0$, $\ramp$ approaches the step
function.  Using the ramp function, we define the following non-convex
surrogate loss function
\begin{equation} \label{ramp:eq:surr}
  \SLS(\vec w)  = \E_{(\vec x, y) \sim \D} \left[ \ramp\left(-y \frac{\dotp{\vec w}{\vec x}}{\snorm{2}{\vec w}} \right)\right] \;.
\end{equation}
\usetikzlibrary{patterns}

\begin{figure}
  \begin{minipage}[t]{0.47\textwidth}

	\centering
	\begin{tikzpicture}[scale=1]

\draw[fill=red, opacity=0.3] (0,0) -- (0,3) arc (90:48:3.0cm) -- cycle;
	\draw[fill=red, opacity=0.3] (0,0) -- (0,-3) arc (270:228:3.0cm) -- cycle;
	\draw[fill=blue, opacity=0.4] (0,0) -- (0,-3) arc (270:408:3.0cm) -- cycle;
	\draw[fill=blue, opacity=0.4] (0,0) -- (0,3) arc (90:228:3.0cm) -- cycle;

	\draw[black, dashed, thick](-2.75,1) -- (2.75,1);
	\draw[black, dashed, thick](-2.75,-1) -- (2.75,-1);
	\coordinate (start) at (0.5,0);
	\coordinate (center) at (0,0);
	\coordinate (end) at (0.5,0.5);
	\draw[->] (-3.2,0) -- (3.2,0) node[anchor=north west,black] {$\vec e_1 $};
	\draw[->] (0,-3.2) -- (0,3.2) node[anchor=south east] {$\vec e_2$};
	\draw[thick,->] (0,0) -- (-0.7,0.7) node[anchor= south east,below,left=0.1mm] {$\wstar$};
	\draw[blue] (-2,-2.22) -- (2,2.22);
	\draw[thick ,->] (0,0) -- (0,1) node[right=2mm,below] {$\bw$};

\pic [draw, <->,
	angle radius=8mm, angle eccentricity=1.2,
	"$\theta$"] {angle = start--center--end};
\end{tikzpicture}
		\caption{The sign of the two-dimensional gradient projection. \label{fig:2d_gradient_sign}}
	\label{fig:2d_gradient_sign2}
      \end{minipage}
	\begin{minipage}[t]{0.54\textwidth}

	\centering
	\begin{tikzpicture}[scale=1]
\coordinate (start) at (0.5,0);
	\coordinate (center) at (0,0);
	\coordinate (end) at (0.5,0.5);

\draw[black,dashed, thick](-3.75,1) -- (3.75,1);
	\draw[black,dashed, thick](-3.75,-1) -- (3.75,-1);
\draw[fill=red, opacity=0.3] (0,0) rectangle (0.9 ,1);
	\draw[fill=red, opacity=0.3] (0,0) rectangle (-0.9,-1);
\draw[fill=blue, opacity=0.4] (-1.5,-1) rectangle (-3.5,1);
	\draw[fill=blue, opacity=0.4](1.5,1) rectangle (3.5,-1);
\draw[<->] (1.3,0) -- (1.3,1) node[black,right=5mm,below=1mm] {$\sigma/2$};
\draw[->] (-3.8,0) -- (3.8,0) node[anchor=north west,black] {$\vec e_1$};
	\draw[->] (0,-3.2) -- (0,3.2) node[anchor=south east] {$\vec e_2$};
	\draw[thick,->] (0,0) -- (-0.7,0.7) node[anchor= south east,below,left=0.1mm] {$\wstar$};
\draw[blue] (-2,-2.22) -- (2,2.22);
\draw[thick ,->] (0,0) -- (0,1) node[right=2mm,below] {$\bw$};
	\pic [draw, <->,
	angle radius=8mm, angle eccentricity=1.2,
	"$\theta$"] {angle = start--center--end};
	\node[] at (1.5,-0.3) {$R/2$};
	\draw[] (1.5,-0.1)--(1.5,0.1);
		\draw[] (3.5,-0.1)--(3.5,0.1);
			\draw[] (-1.5,-0.1)--(-1.5,0.1);
		\draw[] (-3.5,-0.1)--(-3.5,0.1);
	\node[] at (3.5,-0.3) {$R$};
		\node[] at (-1.5,-0.3) {$R/2$};
	\node[] at (-3.5,-0.3) {$R$};
\end{tikzpicture}
	\caption{The ``good" (blue) and ``bad" (red) regions inside a band of size $\sigma$.}
	\label{fig:integration_regions}
      \end{minipage}
\end{figure}
To simplify notation, we will denote the inner product of $\vec x$
and the normalized $\vec w$ as $\ell(\bw, \bx) =
\frac{\dotp{\bw}{\bx}}{\snorm{2}{\bw}}$.  By a straightforward calculation
(see Appendix~\ref{app:gradient_formula}), we get that the gradient of the
objective $\SLS(\bw)$ is
\begin{align}
  \nabla_{\vec w} \SLS(\vec w)
& = \E_{\bx \sim \D_{\bx}} \left[- \ramp'\left( \ell(\bw, \bx)  \right) \ \nabla_{\bw} \ell(\bw, \bx) \ (1- 2 \eta(\bx))\ \sign(\dotp{\wstar}{\vec x}) \right] \;.
\end{align}
Our goal is to establish a claim along the following lines.
\begin{claim}[Informal] \label{clm:non_convex_gradient}
For every $\eps > 0$  there exists $\sigma>0$ such that for any vector
$\wh{\vec w}$ with $\theta(\vec w^{\ast}, \widehat{\vec w}) > \eps$, it holds
$\snorm{2}{\nabla_w \SLS(\wh{\bw})} \geq \eps$.
\end{claim}
The contrapositive of this claim implies that for every $\eps$ we can tune
the parameter $\sigma$ so that all points with sufficiently small gradient have
angle at most $\eps$ with the optimal halfspace $\vec w^{\ast}$.   This is
a parameter distance guarantee that is easy to translate to
missclafication error (using Claim~\ref{lem:angle_zero_one}).

Since it suffices to prove that the norm of the gradient of any ``bad"
hypothesis (i.e., one whose angle with the optimal is greater than $\eps$) is large,
we can restrict our attention to any subspace and bound from below the norm of the
gradient in that subspace.  Let $V = \mrm{span}(\bw^{\ast}, \bw)$ and note that
the inner products $\dotp{\bw^{\ast}}{\bx}$, $\dotp{\bw}{\bx}$
do not change after the projection to this subspace.
Write any point $\bx \in \R^d$ as $\vec v + \vec u$, where $\vec v \in V$ is the
projection of $\vec x$ onto $V$ and $\vec u \in V^{\perp}$.  Now, for each
$\vec v$, we pick the worst-case $\vec u$ (the one that minimizes the norm of
the gradient).  We set $\eta_V(\vec v) = \eta_V(\vec v +  \vec u(\vec v))$.
Since $\eta(\vec x) \leq \eta$ for all $\vec x$, we also have that $\eta_V(\vec
v) \leq \eta$, for all $\vec v \in V$.  Therefore, we have
\begin{equation*}
  \snorm{2}{\nabla_{\bw}\SLS(\bw)} \geq \snorm{2}{\proj_{V} \nabla_{\bw} \SLS(\bw)}
  = \snorm{2}{\E_{(\bx, y) \sim \D_V} [ \nabla_{\bw} \SLS(\bw) ]}\;.
\end{equation*}
Without loss of generality, assume that $\wh{\bw} = \vec e_2$ and $\wstar =
-\sin\theta \cdot \vec e_1 + \cos \theta \cdot \vec e_2$, see
Figure~\ref{fig:2d_gradient_sign}.
To simplify notation, in what follows we denote by $\eta(\bx)$ the function
$\eta_V(\bx)$ after the projection.
Observe that the gradient is always perpendicular to $\wh{\bw} = \vec e_2$
(this is also clear from the fact that $\SLS(\bw)$ does not depend on the length of $\bw$).
Therefore,
\begin{equation} \label{ramp:eq:2d_gradient_absolute_value}
  \snorm{2}{\E_{(\bx, y) \sim \D_V} [\nabla_{\bw} \SLS(\wh{\bw})]} =
  | \dotp{\nabla_{\bw} \SLS(\wh{\bw})}{\vec e_1} | =
  \left| \E_{\bx \sim (\D_{\bx})_V}[ - \ramp'(\bx_2) (1- 2 \eta(\bx)) \sgn(\dotp{\wstar}{\bx}) \bx_1 ] \right| \;.
\end{equation}
We partition $\R^2$ in two regions according to the sign of
the pointwise gradient $$g(\vec x) = - \ramp'(\bx_2) (1- 2 \eta(\bx)) \sgn(\dotp{\wstar}{\bx}) \bx_1 \;.$$
Let
$$G = \{\vec x \in \R^2: g(\vec x) \geq 0\} =
\{\vec x \in \R^2 : \bx_1 \sgn(\dotp{\vec w^{\ast}}{\vec x}) \leq 0 \} \;,$$
and let $G^c$ be its complement.  See Figure~\ref{fig:2d_gradient_sign} for an illustration.
To give some intuition behind this definition, imagine we were using SGD
in this $2$-dimensional setting, and at some step $t$
we have $\vec w^{(t)} = \widehat{\vec w} = \vec e_2$.
We draw a sample $(\vec x, y)$ from the distribution $\D$ and update the hypothesis.
Then the expected update (with respect to the label $y$) is
$$
\vec w^{(t+1)} = \vec e_2 -  \dotp{g(\vec x)}{\vec e_1} \vec e_1 \;.
$$
Therefore, assuming that $\theta(\vec w^{\ast}, \vec e_2) \in (0, \pi/2)$, the
``good" points (region $G$) are those that decrease the $\vec e_1$ component
(i.e., rotate the hypothesis counter-clockwise) and the ``bad" points (region
$G^c$) are those that try to increase the $\vec e_1$ component (rotate the
hypothesis clockwise); see Figure~\ref{fig:2d_gradient_sign}.

We are now ready
to explain the main idea behind the choice of the ramp function $\ramp(t)$.
Recall that the derivative of the ramp function is the (scaled) indicator of a
band of size $\sigma/2$ around $0$, $r'_\sigma(t) = (1/\sigma) \1\{|t| \leq
\sigma/2\}$.  Therefore, the gradient of this loss function amplifies the
contribution of points close to the current guess $\vec w$, that is,
points inside the band $\1\{|\bx_2| \leq \sigma/2\}$ in our $2$-dimensional example of
Figure~\ref{fig:2d_gradient_sign}.  Assume for simplicity that the marginal
distribution $\mathcal{D}_{\bx}$ is the uniform distribution on the $2$-dimensional
unit ball.  Then, no matter how small the angle of the true halfspace and our
guess $\theta(\vec w^{\ast}, \widehat{\vec w})$ is, we can always pick $\sigma$
sufficiently small so that the contribution of the ``good" points (blue region in
Figure~\ref{fig:2d_gradient_sign}) is much larger than the contribution of the
``bad" points (red region).

Crucial in this argument is the fact that the distribution is ``well-behaved'' in the
sense that the probability of every region is related to its area.
This is where Definition~\ref{def:bounds} comes into play.  To bound from below the
contribution of ``good" points, we require the anti-anti-concentration property of
the distribution, namely a lower bound on the density function (in some bounded
radius). To bound from above the contribution of ``bad" points, we need the
anti-concentration property of Definition~\ref{def:bounds}, namely that the
density is bounded from above (recall that we wanted the
probability of a region to be related to its area).

We are now ready to show that our ramp-based non-convex loss works
for all distributions satisfying Definition~\ref{def:bounds}.
In the following lemma, we prove that we can tune
the parameter $\sigma$ so that the stationary points of our non-convex
loss are close to $\vec w^{\ast}$.  The following lemma is a precise version
of our initial informal goal, Claim~\ref{clm:non_convex_gradient}.

\begin{lemma}[Stationary points of $\SLS$ suffice]\label{lem:structural_massart-ramp}
Let $\D_{\bx}$ be a $(U, R)$-bounded distribution on $\R^d$,
and $\eta<1/2$ be an upper bound on the Massart noise rate.
Fix any $\theta \in (0, \pi/2)$. Let $\wstar \in \Sp^{d-1}$ be the
normal vector to the optimal halfspace and $\wh{\bw} \in \Sp^{d-1}$ be
such that $\theta(\wh{\bw}, \wstar) \in (\theta, \pi -\theta)$.  For
$\sigma \leq \frac{R}{2U} \sqrt{1-2\eta} \sin \theta$,
we have that
$\snorm{2} {\nabla_{\bw} \SLS(\wh{\bw})} \geq (1/8)  R^2 (1-2\eta)/U$.
\end{lemma}
\begin{proof}
We will continue using the notation introduced in the above discussion.  We
let $V$ be the $2$-dimensional subspace spanned by $\vec w^{\ast}$ and
$\widehat{\vec w}$. To simplify notation, we again assume without loss of
generality that $\vec w^{\ast} = -\sin \theta\ \vec e_1 + \cos \theta\ \vec e_2$ and
$\widehat{\vec w} = \vec e_2$, see Figure~\ref{fig:2d_gradient_sign}.  Using
the triangle inequality and Equation~\eqref{ramp:eq:2d_gradient_absolute_value},
we obtain
\begin{align}
  \snorm{2}{\E_{(\bx,y) \sim \D_V} [\nabla_{\bw} \SLS(\wh{\bw})]}
&\geq \E_{\bx \sim (\D_{\bx})_V} \left[\ramp'(\bx_2) (1- 2 \eta(\bx)) |\bx_1| \1_G(\bx)\right]
- \E_{\bx \sim(\D_{\bx})_V} \left[\ramp'(\bx_2) (1- 2 \eta(\bx)) |\bx_1| \1_{\bx \in G^c} \right] \nonumber \\
\nonumber
\\
&= \E_{\bx \sim (\D_{\bx})_V} \left[\ramp'(\bx_2) (1- 2 \eta(\bx)) |\bx_1|\right]
- 2 \E_{\bx \sim(\D_{\bx})_V} \left[\ramp'(\bx_2) (1- 2 \eta(\bx)) |\bx_1|  \1_{\bx \in G^c} \right] \label{ramp:eq:2d_gradient_difference_lower_bound} \;.
\end{align}
We now bound from below the first term, as follows
\begin{align}
  \label{ramp:eq:good_lower_bound}
\E_{\bx \sim (\D_{\bx})_V} \left[\ramp'(\bx_2) (1- 2 \eta(\bx)) |\bx_1| \right] \nonumber
&\geq (1-2\eta) \E_{\bx \sim (\D_{\bx})_V} \left[ \frac{\1\{|\bx_2| \leq \sigma/2\}}{\sigma} |\bx_1|
\right] \nonumber \\
  &\geq \frac{(1-2\eta) R}{2\sqrt{2} \sigma} \E_{\bx \sim(\D_{\bx})_V}
  \left[
  \1\left\{ |\bx_2| \leq \frac{\sigma}{2},\ \frac{R}{2\sqrt{2}} \leq |\bx_1| \leq \frac{R}{\sqrt{2}}\right\}
  \right] \nonumber \\
  &\geq \frac{(1-2\eta) R}{2\sqrt{2} \sigma} \cdot \frac{ R \sigma}{\sqrt{2} U} = \frac{ R^2}{4 U}(1-2 \eta),
\end{align}
where the first inequality follows from the upper bound on the noise $\eta(\vec x) \leq \eta$,
and the third one from the lower bound on the $2$-dimensional
density function $1/U$ inside the ball $\snorm{2}{\vec x} \leq R$ (see Definition~\ref{def:bounds}).

We next bound from above the second term of Equation~\eqref{ramp:eq:2d_gradient_difference_lower_bound},
that is the contribution of ``bad" points. We have that
\begin{align*}
\E_{\bx \sim (\D_{\bx})_V} \left[\ramp'(\bx_2) (1- 2 \eta(\bx)) |\bx_1| \1_{\bx \in G^c} \right] \nonumber
&\leq \E_{\bx \sim (\D_{\bx})_V} \left[\frac{\1\{|\bx_2| \leq \sigma/2\}}{\sigma} |\bx_1| \1\{\vec x \in G^c \} \right] \nonumber\\
&\leq \frac{1}{\sigma} \E_{\bx \sim(\D_{\bx})_V} \left[|\bx_1| \1\{\vec x \in G^c, |\bx_2| \leq \sigma/2 \} \right]\;.
\end{align*}
We now observe that for $\theta \in (0, \pi/2]$ it holds
\begin{align*}
G^c = \{\vec x : \bx_1 \sgn(\dotp{\vec w^{\ast}}{\vec x}) > 0 \}
= \{\vec x : \bx_1 \sgn(- \bx_1 \sin \theta + \bx_2 \cos \theta ) > 0 \}
\subseteq \{\vec x : \bx_1 \bx_2 > 0 \}\;.
\end{align*}
On the other hand, if $\theta \in (\pi/2, \pi]$ we have
$ G^c \subseteq \{\vec x: \bx_1 \bx_2 < 0\} $.
Assume first that $\theta \in (0, \pi/2]$ (the same argument works also for the other case).
Then the intersection of the band $\{\vec x: |\bx_2| \leq \sigma/2\}$ and $G^c$
is contained in the union of two rectangles
 ${\cal R} = \{\vec x : |\bx_1| \leq \sigma/(2 \tan \theta),\ |\bx_2| \leq \sigma/2,\ \bx_1 \bx_2 > 0 \}$,
 see Figure~\ref{fig:integration_regions}.  Therefore,
\begin{align}
\E_{\bx \sim (\D_{\bx})_V} \left[\ramp'(\bx_2) (1- 2 \eta(\bx)) |\bx_1| \1_{\bx \in G^c} \right] \nonumber
 &\leq \frac{1}{\sigma} \frac{\sigma}{2 \tan \theta} \E_{\bx \sim (\D_{\bx})_V} \left[
 \1\{\vec x \in G^c, |\bx_1| \leq \frac{\sigma}{2 \tan\theta}, |\bx_2| \leq \frac{\sigma}{2} \} \right] \nonumber \\
 &\leq \frac{1}{\sigma} \frac{\sigma}{2 \tan \theta} \E_{\bx \sim (\D_{\bx})_V} \left[
  \1\{\vec x \in R\} \right] \leq \frac{1}{2 \tan \theta} \cdot \frac{ U \sigma^2 }{ 2 \tan \theta}
\nonumber
\\
&\leq \frac{ R^2}{16 U}(1- 2\eta)
  \label{ramp:eq:bad_upper_bound}\;,
\end{align}
where for the last inequality we used our assumption that
$\sigma \leq \frac{ R}{2 U} \sqrt{1-2 \eta} \sin \theta$.
To finish the proof, we substitute the bounds \eqref{ramp:eq:good_lower_bound},
\eqref{ramp:eq:bad_upper_bound} in Equation \eqref{ramp:eq:2d_gradient_difference_lower_bound}.
\end{proof}

\subsection{Main structural result: Non-convex surrogate via smooth approximation}\label{ssec:smooth-sigmoid}
In this subsection, we prove the structural result that is required
for the correctness of our efficient gradient-descent algorithm in the following section.
We consider the non-convex surrogate loss
\begin{equation} \label{eq:surr}
\SL(\vec w)  = \E_{(\vec x, y) \sim \D} \left[ S_{\sigma}\left(-y \frac{\dotp{\vec w}{\vec x}}{\snorm{2}{\vec w}} \right)\right],
\end{equation}
where $S_\sigma(t) = \frac{1}{1 + e^{-t/\sigma}}$ is the logistic function
with growth rate $1/\sigma$.  That is, we have replaced the step function by
the sigmoid.  As $\sigma \to 0$, $S_\sigma(t)$ approaches the step function.
Formally, we prove the following:

\begin{lemma}[Stationary points of $\SL$ suffice]\label{lem:structural_massart}
Let $\D_{\bx}$ be a $(U, R)$-bounded distribution on $\R^d$,
and $\eta<1/2$ be an upper bound on the Massart noise rate.
Fix any $\theta \in (0, \pi/2)$. Let $\wstar \in \Sp^{d-1}$ be the
normal vector to the optimal halfspace and $\wh{\bw} \in \Sp^{d-1}$ be
such that $\theta(\wh{\bw}, \wstar) \in (\theta, \pi -\theta)$.  For
$\sigma \leq \frac{R}{8 U} \sqrt{1-2\eta} \sin \theta$,
we have that
$\snorm{2} {\nabla_{\bw} \SL(\wh{\bw})} \geq \frac{1}{32U} R^2  (1-2\eta)$.
\end{lemma}

The proof of Lemma~\ref{lem:structural_massart} is conceptually similar to the proof of Lemma~\ref{lem:structural_massart-ramp}
for the ramp function given in the previous subsection. The main difference is that, in the smoothed setting,
it is harder to bound the contribution of each region of Figure~\ref{fig:2d_gradient_sign}
and the calculations end-up being more technical.

\begin{proof}[Proof of Lemma~\ref{lem:structural_massart}]

\begin{figure}
	\centering
	\begin{tikzpicture}[scale=1]
\coordinate (start) at (0.5,0);
	\coordinate (center) at (0,0);
	\coordinate (end) at (0.5,0.5);

\draw[fill=blue, opacity=0.4] (0,0) -- (-2,0) arc (180:90:2.0cm) -- cycle;
\draw[fill=blue, opacity=0.4] (0,0) -- (2,0) arc (360:270:2.0cm) -- cycle;
\draw[fill=red, opacity=0.4] (0,0) -- (2,2.22) arc (48:90:3.0cm) -- cycle;
\draw[fill=red, opacity=0.4] (0,0) -- (-2,-2.22) arc (228:270:3.0cm) -- cycle;
\draw[->] (-4,0) -- (4,0) node[anchor=north west,black] {$\vec e_1$};
	\draw[->] (0,-3.2) -- (0,3.2) node[anchor=south east] {$\vec e_2$};
	\draw[thick,->] (0,0) -- (-0.7,0.7) node[anchor= south east] {$\wstar$};
	\draw[black] (-2,-2.22) -- (2,2.22);
	\draw[thick ,->] (0,0) -- (0,1) node[right] {$\bw$};
	\pic [draw, <->,
	angle radius=8mm, angle eccentricity=1.2,
	"$\theta$"] {angle = start--center--end};
\node[] at (-2,-0.3) {$R$};
\node[] at (2,0.3) {$R$};
\end{tikzpicture}
	\caption{The ``good" (blue) and ``bad" (red) regions.}
	\label{fig:integration_regions_sigm}
\end{figure}
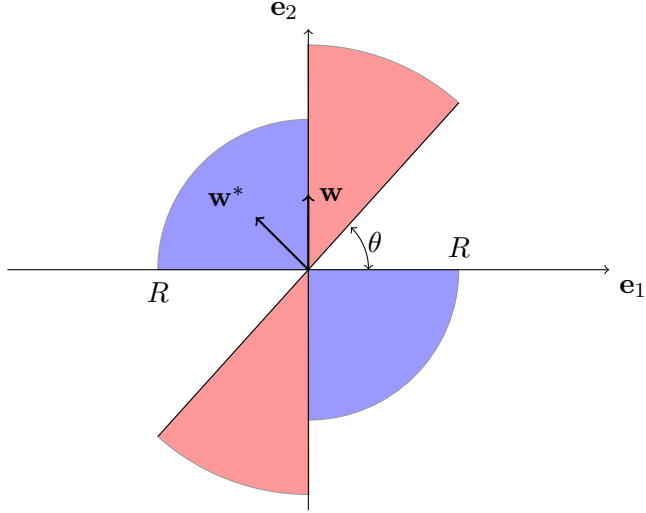

Without loss of generality, we will assume that $\wh{\bw} = \vec e_2$ and
$\wstar = -\sin\theta \cdot \vec e_1 + \cos \theta \cdot \vec e_2$.
Using the same argument as in the proof of Section~\ref{ssec:ramp},
we let $V = \mrm{span}(\bw^{\ast}, \bw)$ and have
\begin{equation} \label{eq:2d_gradient_absolute_value}
\snorm{2}{\E_{(\bx, y) \sim \D_V} [\nabla_{\bw} \SL(\wh{\bw})]} =
\left| \E_{\bx \sim (\D_{\bx})_V}[ - S'_{\sigma}(|\bx_2|) (1- 2 \eta(\bx)) \sgn(\dotp{\wstar}{\bx}) \bx_1 ] \right| \;.
\end{equation}
We partition $\R^2$ in two regions according to the sign of the gradient.
Let
$$G = \{(\bx_1, \bx_2) \in \R^2 : \bx_1 \sgn(\dotp{\vec w^{\ast}}{\vec x}) > 0 \} \;,$$
and let $G^c$ be its complement. Using the triangle inequality
and Equation \eqref{eq:2d_gradient_absolute_value}, we obtain
\begin{align}
\snorm{2}{\E_{(\bx,y) \sim \D_V} [\nabla_{\bw} \SL(\wh{\bw})]}
&\geq \E_{\bx \sim (\D_{\bx})_V} \left[S'_{\sigma}(|\bx_2|) (1- 2 \eta(\bx)) |\bx_1| \1_G(\bx)\right]
   - \E_{\bx \sim (\D_{\bx})_V} \left[S'_{\sigma}(|\bx_2|) (1- 2 \eta(\bx)) |\bx_1| \1_{G^c}(\bx)\right] \nonumber \\
&\geq \frac{(1-2\eta)}{4} \E_{\bx \sim (\D_{\bx})_V} \left[ \frac{e^{-|\bx_2|/\sigma}}{\sigma} \cdot |\bx_1| \cdot \1_G(\bx)\right]
   - \E_{\bx \sim (\D_{\bx})_V} \left[\frac{e^{-|\bx_2|/\sigma}}{\sigma} \cdot |\bx_1| \cdot \1_{G^c}(\bx)\right] \;,
   \label{eq:2d_gradient_difference_lower_bound}
 \end{align}
where we used the upper bound on the Massart noise rate $\eta(\bx) \leq \eta$
and the fact that the sigmoid $S_{\sigma}(|t|)^2$ is bounded from above by $1$
and bounded from below by $1/4$.

We can now bound each term separately
using the fact that the distribution is $(U,R)$-bounded.
Assume first that $\theta(\wstar, \wh{\bw}) = \theta \in (0, \pi/2)$.
Then we can express the region in polar coordinates as
$G = \{ (r, \phi) : \phi \in (0, \theta) \cup (\pi/2, \pi +\theta) \cup (3 \pi/2, 2 \pi) \}$.
See Figure~\ref{fig:integration_regions_sigm} for an illustration.

We denote by $\gamma(x, y)$ the density of the $2$-dimensional
projection on $V$ of the marginal distribution $\D_{\bx}$.  Since the integral
is non-negative, we can bound from below the contribution of region $G$ on
the gradient by integrating over $\phi \in (\pi/2, \pi)$. Specifically, we have:
\begin{align}
\E_{\bx \sim(\D_{\bx})_V} \left[\frac{e^{-|\bx_2|/\sigma}}{\sigma}\ |\bx_1|\ \1_G(\bx)\right]
&\geq \int_{0}^{\infty} \int_{\pi/2}^{\pi} \gamma(r \cos\phi,r \sin\phi)r^2 |\cos\phi|  \frac{\sigMO{r \sin{\phi}}}{\sigma} \d\phi \d r \nonumber\\
&= \int_{0}^{\infty} \int_{0}^{\pi/2} \gamma(r \cos\phi,r \sin\phi)r^2 \cos\phi  \frac{\sigMO{r \sin{\phi}}}{\sigma} \d\phi \d r \nonumber\\
&\geq \frac{1}{U} \int_{0}^{R} r^2 \d r \int_0^{\pi/2} \cos\phi \frac{\sigMO{R \sin{\phi}}}{\sigma} \d\phi \nonumber   \\
&= \frac{1}{3 U} R^2 \left(1-e^{-\frac{R}{\sigma }}\right) \geq \frac{1}{3 U} R^2 \left(1-e^{-8}\right) \label{eq:sec3Good}\;,
\end{align}
where for the second inequality we used the lower bound $1/U$ on the density
function $\gamma(x,y)$ (see Definition~\ref{def:bounds}) and for the last inequality we used that $\sigma \leq \frac{R}{8}$.

We next bound from above the contribution of the gradient in region $G^c$.
Note that $G^c = \{(r, \phi): \phi \in B_\theta = (\pi/2-\theta, \pi/2) \cup (3 \pi/2 -\theta, 3 \pi/2)\}$.
Hence, we can write:
\begin{align}
\E_{\bx \sim(\D_{\bx})_V} \left[\frac{e^{-|\bx_2|/\sigma}}{\sigma}\ |\bx_1|\ \1_{G^c}(\bx)\right]
&=  \int_{0}^{\infty} \int_{\phi \in B_{\theta}} \density(r \cos \phi,r \sin \phi)r^2 \cos{\phi} \sigMO{r \sin{\phi}} \d\phi \d r\nonumber \\
&\leq  \frac{2U}{\sigma}\int_{0}^{\infty}\int_{\theta}^{\pi/2} r^2 \cos{\phi} \sigMO{r \sin{\phi}} \d\phi \d r \nonumber\\
&= \frac{2U \sigma ^2 \cos^2 \theta}{\sin^2 \theta} \nonumber\\
&= \frac{(1-2\eta) R^2 }{32 U} \cos^2 \theta \label{eq:sec3Bad} \;,
\end{align}
where the inequality follows from the upper bound $U$ on the density
$\gamma(x,y)$ (see Definition~\ref{def:bounds}) and the last inequality follows from our assumption that
$\sigma \leq \frac{R}{8 U}\sqrt{ 1-2\eta}  \sin(\theta)$.
Combining \eqref{eq:sec3Good} and \eqref{eq:sec3Bad}, we have
\begin{align}
\E_{\bx \sim(\D_{\bx})_V} \left[\frac{e^{-|\bx_2|/\sigma}}{\sigma}\ |\bx_1|\ \1_{G^c}(\bx)\right] &\leq \frac{(1-2\eta) R^2  }{32 U}\cos^2 \theta\nonumber \\&\leq \frac{(1-2\eta) R^2  \left(1-e^{-8}\right)}{24 U}\nonumber\\ &\leq \frac{1}{2}\frac{(1-2\eta)}{4}\E_{\bx \sim(\D_{\bx})_V} \left[\frac{e^{-|\bx_2|/\sigma}}{\sigma}\ |\bx_1|\ \1_G(\bx)\right]\;, \label{eq:bound_good_bads}
\end{align}
where the second inequality follows from $\cos^2\theta\leq 1$ and $\frac{1}{32}\leq \frac{ \left(1-e^{-8}\right)}{24}$.
Using \eqref{eq:bound_good_bads} in \eqref{eq:2d_gradient_difference_lower_bound}, we obtain
$$\snorm{2}{\E_{(\bx, y) \sim \D_V}\left[ \nabla_{\bw} \SL(\wh{\bw}) \right]}
\geq \frac{1}{2}\frac{(1-2\eta)}{4}\E_{\bx \sim(\D_{\bx})_V} \left[\frac{e^{-|\bx_2|/\sigma}}{\sigma}\ |\bx_1|\ \1_G(\bx)\right]  \geq
\frac{1}{32 U} (1- 2 \eta)\ R^2 \;.
$$
To conclude  the proof, notice that the case where $\theta(\wh{\bw}, \wstar) \in (\pi/2,\pi-\theta )$ follows similarly.
Finally, in the case where $\theta=\pi/2$, the region $G^c$ is empty,
and we again get the same lower bound on the gradient.
This completes the proof of Lemma~\ref{lem:structural_massart}.
\end{proof}

 \section{Main Algorihtmic Result: Proof of Theorem~\ref{thm:main-inf}} \label{sec:alg}

In this section, we prove our main algorithmic result, which we restate below:

\begin{theorem}\label{thm:main_massart}
Let $\D$ be a distribution on $\R^d \times \{-1, +1\}$ such that the marginal
$\D_{\bx}$ on $\R^d$ is $(U, R, t())$-bounded.  Let $\eta<1/2$ be
an upper bound on the Massart noise rate. Algorithm~\ref{alg:full} has the following
performance guarantee: It draws
$m = O\left( (U/R)^{12} \cdot t^8(\eps/2) / (1-2 \eta)^{10}\right) \cdot O(d/\eps^4) $
labeled examples from $\D$,
uses $O(m)$ gradient evaluations, and outputs a hypothesis vector $\bar{\vec w}$ that satisfies
$\err_{0-1}^{\D_{\bx}}(h_{\bar{\bw}},f)\leq \eps$ with probability at least $1-\delta$, where $f$ is the target
halfspace.
\end{theorem}

Our algorithm proceeds by Projected Stochastic Gradient Descent (PSGD),
with projection on the $\ell_2$-unit sphere,
to find an approximate stationary point of our non-convex surrogate loss.
Since $\SL(\bw)$ is non-smooth for vectors $\bw$ close to $\vec 0$,
at each step, we project the update on the unit sphere
to avoid the region where the smoothness parameter is high.

Recall that a function $f:\R^d \mapsto \R$ is called $L$-Lipschitz
if there is a parameter $L>0$ such that
$\snorm{2}{f(\vec x)-f(\vec y)} \leq L \snorm{2}{\vec x-\vec y}$ for all $\vec x, \vec y \in \R^d$.
We will make use of the following folklore result on the convergence of projected SGD
(for completeness, we provide a proof in Appendix~\ref{app:sgd}).

\begin{algorithm}[H]
  \caption{PSGD for $f(\bw) = \E_{\vec z\sim \D}[g(\vec z, \bw)]$}
  \label{alg:PSGD}
  \begin{algorithmic}[1]
    \Procedure{psgd}{$f, T, \beta$}
    \Comment{$f(\bw) = \E_{\vec z \sim \D}[g(\vec z, \bw)]$: loss,
    $T$: number of steps, $\beta$: step size.}
    \State ${\vec w}^{(0)} \gets \vec e_1$
     \State \textbf{for} $i = 1, \dots, T$ \textbf{do}
    	\State \qquad Sample $\vec z^{(i)}$ from $\D$.
    	\State\qquad  ${\vec v}^{(i)} \gets {\vec w}^{(i-1)} - \beta \nabla_{\vec w} g({\vec z}^{(i)}, {\vec w}^{(i-1)})$
    	\State \qquad ${\vec w}^{(i)} \gets {\vec v}^{(i)}/\snorm{2}{{\vec v}^{(i)}}$
    \State  \textbf{return} $({\vec w}^{(1)}, \ldots, {\vec w}^{(T)})$.
    \EndProcedure
  \end{algorithmic}
\end{algorithm}

 \begin{replemma}{lem:PSGD}[PSGD]
Let $f : \R^d \mapsto \R$ with $f( \bw) = \E_{\vec z\sim \D}[g(\vec z, \vec w)]$
for some function $g:\R^d \times \R^d \mapsto \R$. Assume that for any
vector $\vec w$, $g(\cdot,\bw)$ is positive homogeneous of degree-$0$ on $\bw$.
Let $\mathcal{W} = \{\vec w \in \R^d: \snorm{2}{\vec w} \geq 1 \}$ and assume
that $f, g$ are continuously differentiable functions on $\mathcal{W}$.
Moreover, assume that $|f(\vec w)| \leq R$, $\nabla_{\bw} f(\vec w)$ is
$L$-Lipschitz on $\mathcal{W}$, $\E_{\vec z \sim \D} \left[\snorm{2}{\nabla_{\vec w} g(\vec z, \vec w)}^2\right] \leq B$
for all $\vec w \in \mathcal{W}$.  After $T$ iterations the output
$({\vec w}^{(1)}, \ldots, {\vec w}^{(T)})$ of Algorithm~\ref{alg:PSGD} satisfies
\[\E_{{\vec z}^{(1)}, \ldots, {\vec z}^{(T)} \sim \D}
   \lp[ \frac{1}{T} \sum_{i=1}^T \snorm{2}{\nabla_{\vec w} f({\vec w}^{(i)})}^2 \rp] \leq \sqrt{\frac{LBR}{2 T}} \;.
\]
If, additionally, $\snorm{2}{\E_{\vec z \sim \D} [\nabla_{\vec w} g(\vec z, \vec w)]}^2 \leq C$
for all $\vec w \in \mathcal{W}$, we have that with
$T = (2 L B R+ 8 C^2 \log(1/\delta))/\eps^4$ it holds
$\min_{i=1,\ldots, T} \snorm{2}{\nabla_{\vec w} f(\vec w^{(i)})} \leq \eps,$
with probability at least $1-\delta$.
\end{replemma}

We will require the following lemma establishing the smoothness properties of our loss
(based on $S_{\sigma}$).  See Appendix~\ref{app:smooth-lemma} for the proof.

\begin{replemma}{lem:sigmoid_smoothness}[Sigmoid Smoothness]
Let $S_\sigma(t) = 1/(1 + e^{-t/\sigma})$ and $\SL(\bw) = \E_{(\vec x,y) \sim
\D} \left[S_{\sigma}\left(-y \frac{\dotp{\vec w}{\vec x}}{\snorm{2}{\vec w}}
\right)\right]$, for $\bw \in \W$, where $\W=\{\bw\in \R^d:
\snorm{2}{\bw}\geq 1\}$.  We have that $\SL(\vec w)$ is continuously
differentiable in $\W$, $|\SL(\vec w)| \leq 1$, $\E_{(\bx, y) \sim
\D}[\snorm{2}{\nabla_{{\vec w}} S_{\sigma}({\vec w}, \vec x, y)}^2] \leq
4d/\sigma^2$, $\snorm{2}{\nabla_{\vec w} \SL(\vec w)}^2 \leq 4/\sigma^2$, and
$\nabla_{\vec w} \SL(\vec w)$ is $(6/\sigma + 12 /\sigma^2)$-Lipschitz.
\end{replemma}

Putting everything together gives Theorem~\ref{thm:main_massart}.

  \begin{algorithm}[H]
    \caption{Learning Halfspaces with Massart Noise}
    \label{alg:full}
    \begin{algorithmic}[1]
      \Procedure{Alg}{$\eps$, $U$, $R$, $t(\cdot)$} \State $C_1\gets \Theta(U^{12}/R^{12})$.
      \State $C_2\gets \Theta(R/U^2)$.
      \State $T \gets C_1\ d\ t(\eps/2)^8/(\eps^4 (1- 2 \eta)^{10})\log(1/\delta)$. \Comment{number of steps}
      \State $\beta \gets C_2^2\ d (1-2\eta)^3 \eps^2 /(t(\eps/2)^4T^{1/2})$. \Comment{step size}
      \State $\sigma \gets  C_2\ \sqrt{1- 2 \eta}\ \eps/t^2(\eps/2)$.
\State $({\vec w}^{(0)}, {\vec w}^{(1)},\ldots, {\vec w}^{(T)}) \gets \mathrm{PSGD}(f, T, \beta)$.
      \Comment{ $f(\vec w) =  \E_{(\bx, y) \sim \D}\left[S_{\sigma}\Big(-y \frac{\dotp{\vec w}{\vec x}}{\snorm{2}{\vec w}} \Big)\right]$, \eqref{alg:PSGD}}
      \State $L \gets\{\pm {\vec w}^{(i)}\}_{i\in [T]}$.\label{alg:list_vec} \Comment{$L$: List of candidate vectors}
      \State Draw $N=O(\log(T/\delta)/(\eps^2 (1-2 \eta)^2))$ samples from $\D$.
      \State $\bar{\vec w} \gets \argmin_{\vec w \in L}
      \sum_{j=1}^N \1\{\sgn(\dotp{\vec w}{\vec x^{(j)}}) \neq y^{(j)}\}$.
      \State \textbf{return} $\bar{\vec w}$.
      \EndProcedure
    \end{algorithmic}
  \end{algorithm}

\begin{proof}[Proof of Theorem~\ref{thm:main_massart}]
By Claim~\ref{lem:angle_zero_one}, to guarantee
$\err_{0-1}^{\D_{\bx}}(h_{\bar{\vec w}},f) \leq \eps$
it suffices to show that the angle $\theta(\bar{\vec w}, \vec \wstar) \leq O(\eps (1-2\eta)/(U t^2(\eps/2))) =: \theta_0$.
Using (the contrapositive of) Lemma~\ref{lem:structural_massart},
we get that with $\sigma = \Theta((R/U) \sqrt{1-2 \eta} \theta_0)$,
if the norm squared of the gradient of some vector $\vec w \in \mathbb{S}^{d-1}$
is smaller than $\rho=O((R^2/U) (1-2 \eta))$, then $\vec w$
is close to either $\vec \wstar$ or $-\vec \wstar$ -- that is, $\theta(\vec w, \vec \wstar) \leq \theta_0$ --
or $\theta(\vec w, -\vec \wstar) \leq \theta_0$.
Therefore, it suffices to find a point $\vec w$ with gradient
$\snorm{2}{\nabla_{\vec w} \SL(\vec w)} \leq \rho$.

From Lemma~\ref{lem:sigmoid_smoothness}, we have that our PSGD objective function
is bounded above by $1$,
$$\E\left[\snorm{2}{\nabla_{\vec w}  S_{\sigma}\Big(-y \frac{\dotp{\vec w}{\vec x}}{\snorm{2}{\vec w}} \Big)}^2\right] \leq O(d/\sigma^2) \;,$$
$\snorm{2}{\E\left[\nabla_{\vec w}  S_{\sigma}\Big(-y \frac{\dotp{\vec w}{\vec x}}{\snorm{2}{\vec w}} \Big)\right] }^2\leq O(1/\sigma^2)$,
and that the gradient is Lipschitz with Lipschitz constant $O(1/\sigma^2)$.
Using these bounds for the parameters of Lemma~\ref{lem:PSGD},
we get that with $T = O(\frac{d}{\sigma^4 \rho^4} \log(1/\delta))$ steps,
the norm of the gradient of some vector in the list $({\vec w}^{(0)}, \ldots, {\vec w}^{(T)})$
will be at most $ \rho$ with probability $1-\delta$.  Therefore,
the required number of iterations is
$$ T = O\left(d\frac{  U^{12}}{R^{12} }  \frac{ t^8(\eps/2) \log(1/\delta)}{\eps^4 (1-2\eta)^{10}}\right)\;.$$
We know that one of the hypotheses in the list $L$ (line \ref{alg:list_vec} of Algorithm~\ref{alg:full})
is $\eps$-close to the true $\vec \wstar$. We can evaluate all of them on a
small number of samples from the distribution $\D$ to obtain the best among them.
From Hoeffding's inequality, it follows that $N = O(\log (T/\delta)/(\eps^2 (1-2 \eta)^2))$
samples are sufficient to guarantee that the excess error
of the chosen hypothesis is at most $\eps (1- 2\eta)$.
Using Fact~\ref{fact:massart_error}, for any hypotheses $h$, and the target concept $f$, it holds
$\err_{0-1}^{\D_{\bx}}(h,f) \leq \frac {1}{(1- 2\eta)} (\err_{0-1}^{\D}(h)-\opt),$
and therefore the chosen hypothesis achieves error at most $2\eps$.
This completes the proof of Theorem~\ref{thm:main_massart}.
\end{proof}
 \section{Strong Massart Noise Model}\label{sec:generalized}

We start by defining the strong Massart noise model, which was considered in~\cite{ZhangLC17} 
for the special case of the uniform distribution on the sphere. The main difference 
with the standard Massart noise model is that, in the strong model, the noise rate is allowed
to approach arbitrarily close to $1/2$ for points that lie very close to the separating hyperplane.

\begin{definition}[Distribution-specific PAC Learning with Strong Massart Noise] \label{def:massart-learning-strong}
Let $\mathcal{C}$ be the concept class of halfspaces over $X= \R^d$, 
$\mathcal{F}$ be a {\em known family} of structured distributions on $X$, $0< c \leq 1$
and $0< \eps <1$.
Let $f(\vec x)=\sign(\dotp{\vec \wstar}{\vec x})$ be an unknown target function in $\mathcal{C}$.
A {\em noisy example oracle}, $\mathrm{EX}^{\mathrm{SMas}}(f, \mathcal{F}, \eta)$,
works as follows: Each time $\mathrm{EX}^{\mathrm{SMas}}(f, \mathcal{F}, \eta)$ is invoked,
it returns a labeled example $(\bx, y)$, such that: (a) $\bx \sim \D_{\bx}$, where $\D_{\bx}$ is a fixed
distribution in $\mathcal{F}$, and (b) $y = f(\bx)$ with probability $1-\eta(\bx)$ 
and $y = -f(\bx)$ with probability $\eta(\bx)$, for an {\em unknown} parameter  $\eta(\vec x)\leq
\max\{1/2 - c |\dotp{\vec \wstar}{\vec x}|, 0 \}$. Let $\D$ denote the joint distribution on $(\bx, y)$ generated by the above oracle.
A learning algorithm is given i.i.d. samples from $\D$ and its goal is to output a hypothesis $h$ 
such that with high probability the misclassification error of $h$ is $\eps$-close to 
the misclassfication error of $f$, i.e., it holds  $\err_{0-1}^{\D}(h) \leq \err_{0-1}^{\D}(f)+ \eps$.
\end{definition}

\usepgfplotslibrary{fillbetween}
\usetikzlibrary{intersections}
\pgfdeclarelayer{bg}
\pgfsetlayers{bg,main}

The main result of this section is the following theorem:

\begin{theorem}[Learning Halfspaces with Strong Massart Noise]\label{thm:generalized_massart}
Let $\D$ be a distribution on $\R^d \times \{\pm1\}$ such that the marginal
$\D_{\bx}$ on $\R^d$ is $( U, R, t())$-bounded. Let $0< c<1$ be the parameter 
of the strong Massart noise model. Algorithm~\ref{alg:full_gen} has the following performance guarantee:
It draws $ m = O\left((U^{12}/R^{18}) (t^8(\eps/2)/c^6)\right) O(d/\eps^4)$
labeled examples from $\D$, uses $O(m)$ gradient evaluations, and outputs a hypothesis
vector $\bar{\vec w}$ that satisfies $\err_{0-1}^{\D}(h_{\bar{\vec w}}) \leq \err_{0-1}^{\D}(f)+ \eps$
with probability at least $1-\delta$.
\end{theorem}   

The proof of Theorem~\ref{thm:generalized_massart} follows along the same lines
as in the previous sections. We show that any stationary point of our non-convex surrogate
suffices and then use projected SGD. 

The main structural result of this section generalizes Lemma~\ref{lem:structural_massart}:

\begin{lemma}[Stationary points of $\SL$ suffice with strong Massart noise]\label{lem:genarized_masart} 
Let $\D_{\vec x}$ be a $(U,R)$-bounded distribution on $\R^d$, 
and let $c \in (0, 1)$ be the parameter of strong Massart noise model.
Let $\theta \in (0, \pi/2)$.  Let $\vec \wstar \in \Sp^{d-1}$ be the normal vector 
to an optimal halfspace and $\wh{\bw}\in \Sp^{d-1}$ be 
such that $\theta(\wh{\bw}, \vec \wstar) \in (\theta, \pi - \theta)$.
For $\sigma \leq \frac{R}{24 U}\sqrt{c R  }  \sin(\theta)$,
we have $\snorm{2} {\nabla_{\vec w} \SL( \wh{\bw})} \geq \frac{1}{288 U} c \ R^3.$
\end{lemma}
\begin{proof}
Without loss of generality, we can assume that $\wh{\bw} = \vec e_2$ and
$\wstar = -\sin\theta \cdot \vec e_1 + \cos \theta \cdot \vec e_2$.
Using the same argument as in the Section \ref{section3}, for $V = \mrm{span}(\bw^{\ast}, \bw)$, we have
		\begin{equation}
		\label{eq:2d_gradient_absolute_value_gen}
		\snorm{2}{
			\E_{(\vec x, y) \sim \D_V}
			[
			\nabla_{\vec w} \SL(\wh{\bw})
			]
		}
		= | \dotp{\nabla_{\vec w} \SL(\wh{\bw})}{\vec e_1} |
		= \left| \E_{\vec x \sim \D_{\bx}}[ - S'_{\sigma}(|\bx_2|)  (1- 2 \eta(\vec x))
		\sgn(\dotp{\vec w^*}{\vec x}) \bx_1 ]
		\right|
		\end{equation}
		
		We partition $\R^2$ in two regions according to the sign of the gradient.  Let
		$G = \{(\bx_1, \bx_2) \in \R^2 : \bx_1 \sgn(\dotp{\vec w^*}{\vec x}) > 0 \}$,
		and let $G^c$ be its complement.  Using the triangle inequality
		and Equation \eqref{eq:2d_gradient_absolute_value_gen} we obtain
		\begin{align}
		\snorm{2}{
			\E_{(\vec x, y) \sim \D_V}
			[
			\nabla_{\vec w} \SL(\wh{\bw})
			]
		}
		\nonumber
		&\geq
		\E_{\vec x \sim \D_{\bx}} [S'_{\sigma}(|\bx_2|) (1- 2 \eta(\vec x)) |\bx_1| \1_G(\vec x)]
		-
		\E_{\vec x \sim \D_{\bx}} [S'_{\sigma}(|\bx_2|) (1- 2 \eta(\vec x)) |\bx_1| \1_{G^c}(\vec x)]
		\nonumber
		\\
		&\geq
		\frac{1}{4}
		\E_{\vec x \sim \D_{\bx}} \left[ (1-2\eta(\vec x))\frac{e^{-|\bx_2|/\sigma}}{\sigma}\ |\bx_1|\ \1_G(\vec x)\right]
		-
		\E_{\vec x \sim  \D_{\bx}} \left[\frac{e^{-|\bx_2|/\sigma}}{\sigma}\ |\bx_1|\ \1_{G^c}(\vec x)\right] \;,
		\label{eq:2d_gradient_difference_lower_bound_gen}
		\end{align}
where we used the fact that the sigmoid $S_{\sigma}(|t|)^2$ is upper bounded by $1$ and lower
		bounded by $1/4$.  
		
We can now bound each term using the fact that the
distribution is $(U,R)$-bounded.  Assume first that $\theta(\vec w^*, \vec w) = \theta \in (0, \pi/2)$.  
Then, (see Figure~\ref{fig:2d_gradient_sign}) we
can express region $G$ in polar coordinates as $G = \{ (r, \phi) : \phi \in (0, \theta) \cup (\pi/2, \pi +\theta) \cup (3 \pi/2, 2 \pi) \}$. 
We denote by $\gamma(x, y)$ the density of the $2$-dimensional
projection on $V$ of the marginal distribution $\D_{\bx}$.  Since the integrand
is non-negative  we may bound from below the contribution of region $G$ on
the gradient by integrating over  $\phi \in (\pi/2, \pi)$.
\begin{align}
\E_{\vec x \sim  \D_{\bx}}
		\left[(1- 2 \eta(\vec x))\frac{e^{-|\bx_2|/\sigma}}{\sigma}\ |\bx_1|\ \1_G(\vec x)\right]
		&\geq \int_{0}^{\infty}
		\int_{\pi/2}^{\pi} (1- 2 \eta(\vec x))\density(r \cos \phi,r \sin \phi)r^2 |\cos\phi|  \frac{\sigMO{r \sin{\phi}}}{\sigma} \d\phi \d r
		\\ \nonumber
		&= \int_{0}^{\infty}
		\int_{0}^{\pi/2} (1- 2 \eta(\vec x))\density(r \cos \phi,r \sin \phi)r^2 \cos\phi  \frac{\sigMO{r \sin{\phi}}}{\sigma} \d\phi \d r
			\\ \nonumber
		&\geq \int_{R/2}^{R}
		\int_{0}^{\pi/2} c |\dotp{\wstar}{\vec x}|\density(r \cos \phi,r \sin \phi)r^2 \cos\phi  \frac{\sigMO{r \sin{\phi}}}{\sigma} \d\phi \d r
		\\ \nonumber
		&\geq c\frac{ R}{6} \int_{R/2}^{R}
		\int_{0}^{\pi/2}\density(r \cos \phi,r \sin \phi)r^2 \cos\phi  \frac{\sigMO{r \sin{\phi}}}{\sigma} \d\phi \d r
		\\ \nonumber
		&\geq c \frac{ R}{6 U} \int_{R/2}^{R} r^2 \d r
		\int_0^{\pi/2} \cos\phi
		\frac{\sigMO{R \sin{\phi}}}{\sigma} \d\phi \nonumber
		\\
		&=
		c\frac{7 }{144 U} R^3 \left(1-e^{-\frac{R}{\sigma }}\right) \geq c\frac{7 }{144 U} R^3 \left(1-e^{-8}\right) \;, \label{eq:sec3Good_gen}
\end{align}
where for the third inequality we used that for $\norm{\vec x}_{2}\geq R/2$, 
we have that $\dotp{\wstar}{\vec x}= \frac{R}{2}(\cos(\theta) + \sin(\theta)) \geq R/6$, 
for the fourth inequality we used the lower bound $1/U$ on the density
function $\density(r \cos\phi,r \sin\phi)$ (see Definition~\ref{def:bounds}), 
and for the last inequality we used that $\sigma\leq R/8$.
		
We next bound from above the contribution of the gradient of region $G^c$.
We have $G^c = \{(r, \phi): \phi \in B_\theta = (\pi/2-\theta, \pi/2) \cup (3 \pi/2 -\theta, 3 \pi/2)\}$
\begin{align}
\E_{\vec x \sim  \D_{\bx}}
		\left[\frac{e^{-|\bx_2|/\sigma}}{\sigma}\ |\bx_1|\ \1_{G^c}(\vec x)\right]
		&=  \int_{0}^{\infty}
		\int_{\phi \in B_\theta}
	\density(r \cos \phi,r \sin \phi)r^2 \cos\phi \sigMO{r \sin{\phi}} \d\phi \d r \ \nonumber
		\\
		&\leq  \frac{2 U}{\sigma}\int_{0}^{\infty}
		\int_{\theta}^{\pi/2} r^2 \cos\phi \sigMO{r \sin{\phi}} \d\phi \d r \nonumber
		\\
		&= \frac{2U \sigma ^2 \cos^2 \theta}{\sin^2 \theta}
		=\frac{2 R^3 c \cos^2\theta}{24^2 U} \;, \label{eq:sec3Bad_gen}
		\end{align}
where the inequality follows from the upper bound $U$ on the density
$\density(r \cos\phi ,r \sin\phi)$ (see Definition~\ref{def:bounds}), and the last equality follows from the value of $\sigma$.
Combining \eqref{eq:sec3Good_gen} and \eqref{eq:sec3Bad_gen}, we have
\begin{align}
\E_{\bx \sim(\D_{\bx})_V} \left[\frac{e^{-|\bx_2|/\sigma}}{\sigma}\ |\bx_1|\ \1_{G^c}(\bx)\right] 
&\leq \frac{2 R^3 c \cos^2\theta}{24^2 U} \nonumber \\
&\leq \frac{1}{8}\frac{ 7 c R^3  \left(1-e^{-8}\right)}{144 U}\nonumber\\ 
&\leq \frac{1}{2}\frac{1}{4}\E_{\bx \sim(\D_{\bx})_V} \left[\frac{e^{-|\bx_2|/\sigma}}{\sigma}\ |\bx_1|\ \1_G(\bx)\right] \;, \label{eq:bound_good_bads_gen}
\end{align}
where the second inequality follows from the identity $\cos^2\theta\leq 1$ 
and $\frac{2}{24^2}\leq\frac{1}{8} \frac{ 7 \left(1-e^{-8}\right)}{144}$.
Using \eqref{eq:bound_good_bads_gen} in \eqref{eq:2d_gradient_difference_lower_bound_gen}, we obtain
$$
		\snorm{2}{
			\E_{(\vec x, y) \sim \D_V}
			[
			\nabla_{\vec w} \SL(\wh{\bw})
			]
		}
		\geq
\frac{1}{8}\E_{\bx \sim(\D_{\bx})_V} \left[\frac{e^{-|\bx_2|/\sigma}}{\sigma}\ |\bx_1|\ \1_G(\bx)\right]
		\geq \frac{c R^3}{288 U} \;.
		$$
		
To conclude the proof, notice that the case where $\theta(\vec w, \vec \wstar) \in
(\pi/2,\pi-\theta )$ follows by an analogous argument.  
Finally, in the case where $\theta=\pi/2$, the region $G^c$ is empty 
and we can again get the same lower bound on the gradient norm.

\end{proof}

  \begin{algorithm}[H]
	\caption{Learning Halfspaces with Strong Massart Noise}
	\label{alg:full_gen}
	\ \begin{algorithmic}[1]
	\Procedure{Alg}{$\eps$, $U$, $R$, $t(\cdot)$} \State $C_1\gets \Theta(U^{12}/R^{18})$.
	\State $C_2\gets \Theta(R^{3/2}/U^2)$.
	\State $T \gets C_1\ d\ t(\eps/2)^8/(\eps^4 c^{6})\log(1/\delta)$. \Comment{number of steps}
     \State $\beta \gets C_2^2\ d\ c^3 \eps^2 /(t(\eps/2)^4T^{1/2})$. 
	\State $\sigma \gets  C_2\ c^{1/2}\ \eps/t^2(\eps/2)$.
\State $({\vec w}^{(0)}, {\vec w}^{(1)},\ldots, {\vec w}^{(T)}) \gets \mathrm{PSGD}(f, T, \beta)$. \Comment{ $f(\vec w) =   \E_{(\bx, y) \sim \D}\left[S_{\sigma}\Big(-y \frac{\dotp{\vec w}{\vec x}}{\snorm{2}{\vec w}} \Big)\right] $, \eqref{alg:PSGD}}
	\State $L \gets\{\pm {\vec w}^{(i)}\}_{i\in [T]}$.\label{alg_gen:list_vec} \Comment{$L$: List of candinate vectors}
	\State Draw $N=O(\log(T/\delta)/\eps^2 )$ samples from $\D$. 
	\State $\bar{\vec w} \gets \argmin_{\vec w \in L}
	\sum_{j=1}^N \1\{\sgn(\dotp{\vec w}{\vec x^{(j)}}) \neq y^{(j)}\}$.
	\State \textbf{return} $\bar{\vec w}$.
	\EndProcedure
\end{algorithmic}
\end{algorithm}

\begin{proof}[Proof of Theorem~\ref{thm:generalized_massart}]
From Claim~\ref{lem:angle_zero_one}, we have that to make the 
$\err_{0-1}^{\D_{\bx}}(h_{\bar{\vec w}},f) \leq \eps$ it suffices 
to prove that the angle $\theta(\bar{\vec w}, \vec \wstar) \leq O(\eps/(U t^2(\eps/2))) =: \theta$.
Using (the contrapositive of) Lemma~\ref{lem:genarized_masart} we get that with
$\sigma \leq \Theta(R/U \sqrt{c R } \theta)$, if the norm squared of the
gradient of some vector $\vec w \in \mathbb{S}^{d-1}$ is smaller than
$\rho=O(R^3  c/U)$, then $\vec w$ is close to either
$\vec \wstar$ or $-\vec \wstar$, that is $\theta(\vec w, \vec \wstar) \leq \theta$
or $\theta(\vec w, -\vec \wstar) \leq \theta$.  Therefore, it suffices to find
a point $\vec w$ with gradient
$\snorm{2}{\nabla_{\vec w} \SL(\vec w)} \leq \rho \;.$

From Lemma~\ref{lem:sigmoid_smoothness}, we have that our PSGD objective function
$\SL(\vec w)$, is bounded by $1$, 
$$\E\left[\snorm{2}{\nabla_{\vec w}  S_{\sigma}\Big(-y \frac{\dotp{\vec w}{\vec x}}{\snorm{2}{\vec w}} \Big)}^2\right] \leq O(d/\sigma^2) \;,$$ 
$\snorm{2}{\E\left[\nabla_{\vec w}  S_{\sigma}\Big(-y \frac{\dotp{\vec w}{\vec x}}{\snorm{2}{\vec w}} \Big) \right]}^2\leq O(1/\sigma^2)$, 
and that the gradient of $\SL(\vec w)$ is Lipschitz with Lipschitz constant $O(1/\sigma^2)$. 
Using these bounds for the parameters of Lemma~\ref{lem:PSGD}, we get that with
$T = O(\frac{d}{\sigma^4\rho^4} \log(1/\delta))$ rounds, the norm of the gradient of some vector of
the list $({\vec w}^{(0)}, \ldots, {\vec w}^{(T)})$ will be at most $ \rho$ with
$1-\delta$ probability.  Therefore, the required number of rounds is
$$ T = O\left(\frac{ U^{12}}{R^{18} }  \frac{d t^8(\eps/2) \log(1/\delta)}{\eps^4 c^6}\right) \;.$$
Now that we know that one of the hypotheses in the list $L$ (line \ref{alg_gen:list_vec} of Algorithm~\ref{alg:full_gen}) is
$\eps$-close to the true $\vec \wstar$, we can evaluate all of them on a
small number of samples from the distribution $\D$ to obtain the best among
them. The fact that $N = O(\log (T/\delta)/(\eps^2))$ samples are
sufficient to guarantee that the excess error of the chosen
hypothesis is at most $\eps $ with probability $1-\delta$ follows directly from Hoeffding's
inequality.  This completes the proof. 
\end{proof}
 
\clearpage

\bibliographystyle{alpha}
\bibliography{allrefs}
\clearpage
\appendix
\section{Omitted Technical Lemmas} \label{app:bounded_distributions}

\subsection{Formula for the Gradient}
\label{app:gradient_formula}

Recall that to simplify notation, we will write $\ell(\bw, \bx) =
\frac{\dotp{\bw}{\bx}}{\snorm{2}{\bw}}$.  Note that $\nabla_{\bw} \ell(\bw,
\bx) = \frac{\bx}{\snorm{2}{\bw}} - \dotp{\bw}{\bx}
\frac{\bw}{\snorm{2}{\bw}^3}$.  The gradient of the objective $\SLS(\bw)$ is
then
\begin{align}
  \nabla_{\vec w} \SLS(\vec w)
&= \E_{(\bx,y) \sim \D} \left[ - \ramp' \left(-y\ \ell(\bw, \bx)\right) \nabla_{\bw} \ell(\bw, \bx) \ y \right]\nonumber\\
& = \E_{(\bx, y) \sim \D} \left[- \ramp'\left(\ell(\bw, \bx) \right) \ \nabla_{\bw} \ell(\bw, \bx) \ y \right] \nonumber\\
& = \E_{\bx \sim \D_{\bx}} \left[- \ramp'\left(\ell(\bw, \bx) \right) \ \nabla_{\bw} \ell(\bw, \bx) \ (\sign(\dotp{\wstar}{\vec x}) (1-\eta(\bx)) -\sign(\dotp{\wstar}{\vec x})\eta(\bx) ) \right] \nonumber\\
& = \E_{\bx \sim \D_{\bx}} \left[- \ramp'\left( \ell(\bw, \bx)  \right) \ \nabla_{\bw} \ell(\bw, \bx) \ (1- 2 \eta(\bx))\ \sign(\dotp{\wstar}{\vec x}) \right] \label{lem3:formula}\;,
\end{align}
where in the second equality we used that the $\ramp'(t)$ is an even function.

\subsection{Proof of Claim~\ref{lem:angle_zero_one}} \label{app:angle-loss}

The following claim relates the angle between two vectors and the zero-one
loss between the corresponding halfspaces under bounded distributions.
\begin{customclm}{\ref{lem:angle_zero_one}}
\textit{
Let $\D_{\vec x}$ be a $(U, R)$-bounded distribution on $\R^d$.
Then for any
$\vec u, \vec v \in \R^d$ we have
\begin{equation} \label{eq:fact1}
(R^2/U) \theta(\vec u, \vec v) \leq \err_{0-1}^{\D_{\bx}}(h_{\vec u},h_{\vec v}) \;.
\end{equation}
Moreover, if $\D$ is $( U, R, t(\cdot))$-bounded, we have that
for any $\eps \in (0,1]$
\begin{equation}\label{eq:fact12}
\err_{0-1}^{\D_{\bx}}(h_{\vec u},h_{\vec v})\leq U t(\eps)^2 \theta(\vec v, \vec u) + \eps \;.
\end{equation}
}
\end{customclm}
\begin{proof}
  Let $V$ be the subspace spanned by $\vec v, \vec u$, and let $(\D_{\vec x})_V$
  be the projection of $\D_{\bx}$ onto $V$.  Since
  $\dotp{\vec v}{\vec x} = \dotp{\vec v}{\proj_V(\vec x)}$
  and
  $\dotp{\vec u}{\vec x} = \dotp{\vec u}{\proj_V(\vec x)}$
  we have
  \begin{align*}
   \err_{0-1}^{\D_{\bx}}(h_{\vec u},h_{\vec v})
    &=
   \err_{0-1}^{(\D_{\bx})_V}(h_{\vec u},h_{\vec v}) \;.
  \end{align*}
  Without loss of generality, we can assume that $V = \mrm{span}(\vec e_1, \vec e_2)$,
  where $\vec e_1,\vec e_2$ are orthogonal vectors of $\R^2$.
  Then from Definition~\ref{def:bounds}, using the fact that $1/U \leq f_V(\vec x)$ for
  all $\vec x$ such that $\snorm{\infty}{\vec x}\leq R$,
  which is also true for all $\vec x$ with $\snorm{2}{\vec x} \leq R$,
  the above probability is bounded below by
  $ \frac {R^2} {U} \theta(\vec u, \vec v)$, which proves \eqref{eq:fact1}.
  To prove \eqref{eq:fact12}, we observe that
  \begin{align*}
    \err_{0-1}^{(\D_{\bx})_V}(h_{\vec u},h_{\vec v})
    &\leq
    \pr_{\vec x \sim (\D_{\bx})_V}[
    \sgn(\dotp{\vec u}{\vec x}) \neq
    \sgn(\dotp{\vec v}{\vec x})
    \text{ and }
    \snorm{2}{\vec x} \leq t(\eps) ]
    +
    \pr_{\vec x \sim (\D_{\bx})_V}[
    \snorm{2}{\vec x} \geq t(\eps) ]
    \\
    &\leq U t(\eps)^2 \theta + \eps.
  \end{align*}
\end{proof}

\subsection{Relation Between Misclassification Error and Error to Target Halfspace}

The following well-known fact relates the misspecification error with respect to $\D$
and the zero-one loss with respect to the optimal halfspace. We include a proof
for the sake of completeness.

\begin{fact} \label{fact:massart_error}
Let $\D$ be a distribution on $\R^d \times \{ \pm 1 \}$, $\eta< 1/2$ be an upper bound
on the Massart noise rate.  Then
if $f(\vec x) = \sign(\dotp{\vec \wstar}{\vec x})$ and $h(\vec x) = \sign(\dotp{\vec u}{\vec x})$  we have
\begin{align*}
\err_{0-1}^{\D_{\bx}}(h,f)\leq \frac{1}{1-2\eta}\left(\err_{0-1}^{\D}(h) -\opt\right)\;.
\end{align*}
\end{fact}

\begin{proof}
We have that
\begin{align*}
\err_{0-1}^{\D}(h)
= \E_{ (\bx,y) \sim \D}[\1\{ h(\bx)\neq f(\bx)\}
&= \E_{ \bx \sim \D_{\bx}}[(1-\eta(\bx)) \1\{ h(\bx)\neq f(\bx)\}]+
\E_{ \bx \sim \D_{\bx}}[\eta(\bx) \1\{ h(\bx)= f(\bx)\}] \\
&= \E_{ \bx \sim \D_{\bx}}[(1-2\eta(\bx)) \1\{ h(\bx)\neq f(\bx)\}] +
      \E_{ \bx \sim \D_{\bx}}[\eta(\bx)] \\
&\geq \E_{ \bx \sim \D_{\bx}}[(1-2\eta) \1\{ h(\bx)\neq f(\bx)\}] +\opt \\
&=(1-2\eta)\ \err_{0-1}^{\D_{\bx}}(h,f)+\opt \;,
\end{align*}
where in the second inequality we used that $\eta(\vec x)\leq\eta$
and $\E_{ \bx \sim \D_{\bx}}[\eta(\bx)] =\opt$.
\end{proof}

\subsection{Log-concave and $s$-concave distributions are bounded} \label{app:lc-sc}

\begin{lemma}[Isotropic log-concave density bounds \cite{lovasz2007geometry}]
  \label{lem:LogConcaveDensityBounds}
  Let $\density$ be the density of any isotropic log-concave distribution on
  $\R^d$.  Then
  \(
  \density (\vec x) \geq 2^{-6d}
  \)
  for all $\vec x$ such that $ 0 \leq \snorm{2}{\vec x} \leq 1/9$.  Furthermore,
  $\density(\vec x) \leq \me\ 2^{8d} d^{d/2} $ for all $\vec x$.
\end{lemma}

We are also going to use the following concentration inequality
providing sharp bounds on the tail probability of isotropic log-concave
distributions.

\begin{lemma}[Paouris' Inequality \cite{Pao06}]\label{lem:Paouris}
  There exists an absolute constant $c > 0$ such that
  if $\D_{\bx}$ is any isotropic log-concave distribution
  on $\R^d$, then for all $t>1$
  it holds
  \[
    \Prob_{\vec x \sim \D_{\bx}}[ \snorm{2}{\vec x} \geq c t \sqrt{d}]
    \leq \exp(-t \sqrt{d}) \;.
  \]
\end{lemma}

\begin{fact}\label{fact:RU-lc}
An isotropic log-concave distribution on $\R^d$ is
$(e 2^{17}, 1/9, c \log(1/\eps) + 2 c)$-bounded,
where $c>0$ is the absolute constant of Lemma~\ref{lem:Paouris}.
\end{fact}
\begin{proof}
Follows immediately from Lemma~\ref{lem:LogConcaveDensityBounds},
Lemma~\ref{lem:Paouris}, and the fact that the marginals of
isotropic log-concave distributions are also isotropic log-concave.
\end{proof}

Now we are going to prove that $s$-concave are also $(U,R,t)$ bounded for all
$s\geq -\frac{1}{2d+3}$. We will require the following lemma:

\begin{lemma}[Theorem 3 \cite{BZ17}]\label{lem:s_con}
Let $\density(\bx)$ be an isotropic $s$-concave distribution density on $\R^d$,
then the marginal on a subspace of $\R^2$ is $\frac{s}{1+(d-2)s}$-concave.
\end{lemma}
\begin{lemma}[Theorem 5 \cite{BZ17}]\label{lem:s_con2}
	Let $\bx$ come from an isotropic distribution over $\R^d$, with $s$-concave density. Then for every $t\geq 16$, we have $$\pr[\snorm{2}{\bx} >\sqrt{d}t]\leq \left(1-  \frac{c s t}{1+ds}\right)^{(1+ds)/s} \;,$$ where $c$ is an absolute constant.
\end{lemma}

\begin{lemma}[Theorem 9 \cite{BZ17}]\label{lem:s_con_1} Let $\gamma:\R^d\rightarrow\R_+$ be an isotropic $s$-concave density. Then

	(a) Let $D(s,d)=(1+\alpha)^{-1/\alpha}\frac{1+3\beta}{3+3\beta}$, where $\beta=\frac{s}{1+(d-1)s}$,  $\alpha=\frac{\beta}{1+\beta}$ and $\zeta=(1+\alpha)^{-\frac{1}{\alpha}}\frac{1+3\beta}{3+3\beta}$. For any $\vec x\in\R^d$ such that $\|\vec x\|\le D(s,d)$, we have $\gamma(\vec x)\ge \left(\frac{\|\vec x\|}{\zeta}((2-2^{-(d+1)s})^{-1}-1)+1\right)^{1/s}\gamma(0)$.

	(b) $\gamma (\vec x)\le \gamma(0)\left[\left(\frac{1+\beta}{1+3\beta}\sqrt{3(1+\alpha)^{3/\alpha}}2^{d-1+1/s}\right)^{s}-1\right]^{1/s}$ for every $\vec x$.

	(c) $(4e\pi)^{-d/2}\left[\left(\frac{1+\beta}{1+3\beta}\sqrt{3(1+\alpha)^{3/\alpha}}2^{d-1+\frac{1}{s}}\right)^{s}-1\right]^{-\frac{1}{s}}<\gamma(0)\le (2-2^{-(d+1)s})^{1/s}\frac{d\Gamma(d/2)}{2\pi^{d/2}\zeta^d}$.

	(d) $\gamma(\vec x)\le (2-2^{-(d+1)s})^{1/s}\frac{d\Gamma(d/2)}{2\pi^{d/2}\zeta^d}\left[\left(\frac{1+\beta}{1+3\beta}\sqrt{3(1+\alpha)^{3/\alpha}}2^{d-1+1/s}\right)^{s}-1\right]^{1/s}$ for every $\vec x$.

\end{lemma}
\begin{lemma}\label{lem:LRU_s}
Any isotropic $s$-concave distribution  on $\R^d$ with $s\geq -\frac{1}{2d+3}$, is $\big(\Theta(1),\Theta(1),c/\eps^{1/6}\big)$-bounded where $c$ is an absolute constant.
\end{lemma}
\begin{proof}
Set $\Gamma=\big((\frac{1+2s}{1+4s}\sqrt{3(1+s/(1+2s))^{(3 +6s)/s}}2^{1+1/s} )^s -1\big)^{1/s}$.
From Lemma~\ref{lem:s_con_1}, we have
\begin{enumerate}
\item For any $\vec x\in \R^2$ such that $\snorm{2}{\vec x}\leq
            (1+\frac{s}{1+2s})^{-\frac{1+2s}{s}}(\frac{1+4s}{3+6s})$, we have
            $\density(\bx)\geq \frac{1}{4e\pi \Gamma}$.
\item For any $\bx\in \R^2$, we have:
$\density(\bx)\leq \frac{(2^{3s+1}-1)^{1/s} (3+6s)^2\Gamma }{4 \pi (1+4s)^2
(\frac{1+3s}{1+2s})^{- \frac{1+2s}{s}}} $.
\end{enumerate}
From Lemma~\ref{lem:s_con}, we have that the marginals  of an isotropic $s$-concave distribution on $\R^d$, on a $2$-dimensional subspace,
are $s'$-concave where $s'=\frac{s}{1+(d-2)s}$. Using $s\geq -\frac{1}{2d+3}$, for $d\geq 3$, we have $s'>-\frac{1}{8}$ and when $d=2$, we have $s'=s\geq -1/7$. Thus, the value of $s'$ is lower bounded by $-1/7$.
To find the values $(U,R)$, we need to find a lower bound and an upper bound on density.
From the expression of $\Gamma$, we observe that for $s'\geq-1/7$ it holds $\Gamma<34\cdot 10^3$. Therefore, we obtain the following bounds
\begin{align*}
\gamma(\vec x) &\geq \frac{1}{4e\pi \Gamma}> \frac{1}{ 10^7} \;,\\
R&= \left(1+\frac{s'}{1+2s'}\right)^{-\frac{1+2s}{s'}}\frac{1+4s'}{3+6s'} \geq 0.065 \;,\\
\gamma(\vec x)&\leq   \frac{(2^{3s'+1}-1)^{1/s'} (3+6s')^2
\Gamma}{4 \pi (1+4s')^2  (\frac{1+3s'}{1+2s'})^{- \frac{1+2s'}{s'}}}< 3.3\cdot 10^7\;,
\end{align*}
where we simplified each expression using the bounds of $s'$.
 From Lemma~\ref{lem:s_con2} we get tail bounds, by taking the appropriate $s'$
 that maximizes the error in the tail bound (which is $s'=-1/7$).
This completes the proof.
\end{proof}
 \section{Omitted Proofs from Section~\ref{sec:alg}} \label{app:opt}
In Section~\ref{app:sgd}, we establish the convergence
properties of projected SGD that we require. Even though this lemma
should be folklore, we did not find an explicit reference.
In Section~\ref{app:smooth-lemma}, we establish the smoothness
of our non-convex surrogate function.

\subsection{Proof of Lemma~\ref{lem:PSGD}} \label{app:sgd}

For convenience, we restate the lemma here.

\begin{customlem}{\ref{lem:PSGD}}[PSGD]
  \textit{
  Let $f : \R^d \mapsto \R$ with $f( \bw) = \E_{\vec z\sim \D}[g(\vec z, \vec w)]$
  for some function $g:\R^d \times \R^d \mapsto \R$. Assume that for any
  vector $\vec w$, $g(\cdot,\bw)$ is positive homogeneous of degree-$0$ on $\bw$.
  Let $\mathcal{W} = \{\vec w \in \R^d: \snorm{2}{\vec w} \geq 1 \}$ and assume
  that $f, g$ are continuously differentiable functions on $\mathcal{W}$.
  Moreover, assume that $|f(\vec w)| \leq R$, $\nabla_{\bw} f(\vec w)$ is
  $L$-Lipschitz on $\mathcal{W}$, $\E_{\vec z \sim \D} \left[\snorm{2}{\nabla_{\vec w} g(\vec z, \vec w)}^2\right] \leq B$
  for all $\vec w \in \mathcal{W}$.  After $T$ iterations the output
  $({\vec w}^{(1)}, \ldots, {\vec w}^{(T)})$ of Algorithm~\ref{alg:PSGD} satisfies
  \[\E_{{\vec z}^{(1)}, \ldots, {\vec z}^{(T)} \sim \D}
    \lp[ \frac{1}{T} \sum_{i=1}^T \snorm{2}{\nabla_{\vec w} f({\vec w}^{(i)})}^2 \rp] \leq \sqrt{\frac{LBR}{2 T}} \;.
  \]
  If, additionally, $\snorm{2}{\E_{\vec z \sim \D} [\nabla_{\vec w} g(\vec z, \vec w)]}^2 \leq C$
  for all $\vec w \in \mathcal{W}$, we have that with
  $T = (2 L B R+ 8 C^2 \log(1/\delta))/\eps^4$ it holds
  $\min_{i=1,\ldots, T} \snorm{2}{\nabla_{\vec w} f(\vec w^{(i)})} \leq \eps,$
  with probability at least $1-\delta$.
}
\end{customlem}

\begin{proof}
  Consider the update ${\vec v}^{(i)} = {\vec w}^{(i-1)} - \beta \nabla g( \vec z^{(i)},\vec
  w^{(i-1)})$ at iteration $i$ of Algorithm \ref{alg:PSGD}.  The projection step on
  the unit sphere (line 6 of Algorithm~\ref{alg:PSGD}) ensures that
  $\snorm{2}{{\vec w}^{(i-1)}} = 1$.  Observe that, since $g(\vec z,{\vec w})$ is
  constant in the direction of ${\vec w}$, we have that $\nabla_{{\vec w}} g(\vec z,\vec
  w^{(i-1)})$ is perpendicular to ${\vec w}^{(i-1)}$.  Therefore, by the Pythagorean theorem,
  $\snorm{2}{{\vec v}^{(i)}}^2 = \snorm{2}{{\vec w}^{(i-1)}}^2 + \beta^2 \snorm{2}{\nabla g(\vec z^{(i)},\vec
  w^{(i-1)})}^2 > 1$ which implies that ${\vec v}^{(i)} \in \W$.
  Observe that the line that connects ${\vec v}^{(i)}$ and ${\vec w}^{(i-1)}$ is also
  contained in $\W$.  Therefore, we have
  \begin{align*}
    f({\vec v}^{(i)}) - f({\vec w}^{(i-1)})
&=\dotp{\nabla_{{\vec w}} f({\vec w}^{(i-1)})}{{\vec v}^{(i)} - {\vec w}^{(i-1)}}
\\
&+
\int_0^1 \dotp{ \nabla_{{\vec w}} f({\vec w}^{(i-1)} + t ({\vec v}^{(i)} - {\vec w}^{(i-1)}))
- \nabla_{{\vec w}} f({\vec w}^{(i-1)})}{({\vec v}^{(i)} - {\vec w}^{(i-1)})} \d t
\\
&\leq
-\beta \dotp{\nabla f({\vec w}^{(i-1)})}{\nabla_{{\vec w}} g(\vec z^{(i)},{\vec w}^{(i-1)})}
+ \frac{\beta^2 L}{2} \snorm{2}{\nabla_{{\vec w}} g(\vec z^{(i)},{\vec w}^{(i-1)})}^2.
  \end{align*}
  Observe now that, since $f$ does not depend on the length of its argument,
  we have $f({\vec v}^{(i)}) = f({\vec w}^{(i)})$ and therefore
  \begin{align*}
    f({\vec w}^{(i)}) - f({\vec w}^{(i-1)})
    \leq
    -\beta \dotp{\nabla f({\vec w}^{(i-1)})}{\nabla_{{\vec w}} g(\vec z^{(i)},{\vec w}^{(i-1)})}
    + \frac{\beta^2 L}{2} \snorm{2}{\nabla_{{\vec w}} g(\vec z^{(i)},{\vec w}^{(i-1)})}^2.
  \end{align*}
  Conditioning on the previous samples $\vec z^{(1)}, \ldots, \vec z^{(i-1)}$ we have
  \begin{align*}
    \E_{\vec z^{(i)}}[f({\vec w}^{(i)}) - f({\vec w}^{(i-1)}) | \vec z^{(1)},\ldots, \vec z^{(i-1)}]
&\leq
\beta \snorm{2}{\nabla_{{\vec w}} f({\vec w}^{(i-1)})}^2
+ \frac{\beta^2 L}{2} \E_{\vec z^{(i)}}\left[\snorm{2}
{\nabla_{{\vec w}} g(\vec z^{(i)},{\vec w}^{(i-1)})}^2\right]
\\
&\leq
-\beta \snorm{2}{\nabla_{{\vec w}} f({\vec w}^{(i-1)})}^2
+ \frac{\beta^2 L B}{2}.
  \end{align*}
  Rearranging the above inequality, taking the average over $T$ iterations and
  using the law of total expectation, we obtain that
  by setting $\beta = \sqrt{2 R/( L B T)}$.
To get the high-probability version, we set
  $$
  S_T(\vec w^{(1)},\ldots,\vec w^{(T)}) = (1/T) \sum_{i=1}^T \snorm{2}{\nabla f(\vec w^{(i)})}^2.
  $$
  Notice that with $T = 2 L B R/\eps^4$ from the previous argument we obtain
  that
  $ \E[S_T(\vec w^{(1)}, \ldots, \vec w^{(T)})] \leq \eps^2/2$.
  Observe that
  \begin{align*}
    \left|S_T(\vec w^{(1)}, \ldots, \vec w^{(i)}, \ldots, \vec w^{(T)})
    -
    S_T(\vec w^{(1)}, \ldots, {\vec w^{(i)}}', \ldots, \vec w^{(T)})
    \right|
&\leq
\frac{\lp|\snorm{2}{\nabla f(\vec w^{(i)})}^2 - \snorm{2}{\nabla f({\vec w^{(i)}}')}^2 \rp|}{T}
\\
&\leq \frac{2 C}{T}\;.
  \end{align*}
  \begin{lemma}[Theorem 2.2 of \cite{DL:01}]\label{lem:bounded_difference}  Suppose that $X_1,\ldots X_d \in \cal X$ are independent random variables, and let $f:{\cal X}^d \mapsto \R$. Let $c_1,\ldots,c_n$ satisfy
    $$\sup_{x_1,\ldots ,x_d , x_i'} |f(x_1,\ldots x_i,\ldots x_d) - f(x_1,\ldots x_i',\ldots x_d) |\leq c_i $$
    for $i\in[d]$. Then
    \begin{align*}
      \pr[ f(X) - \E[f(X)]\geq  t]
      \leq \exp\bigg(- 2t^2 /\sum_{i=1}^d c_i^2\bigg).
    \end{align*}
  \end{lemma}
  \noindent Now using Lemma~\ref{lem:bounded_difference}, we obtain that
  $$
  \pr[ S_T(\vec w^{(1)}, \ldots, \vec w^{(T)}) - \E[ S_T(\vec w^{(1)}, \ldots, \vec w^{(T)})] > t]
  \leq \exp(- t^2 T/(2 C^2)).
  $$
  Choosing $T \geq 2 L BR/\eps^4 + 8 C^2 \log(1/\delta)/\eps^4$ and combining the above bounds,
  gives us that with probability at least $1-\delta$, it holds
  $ S_T(\vec w^{(1)}, \ldots, \vec w^{(T)}) \leq \eps^2$.
  Since the minimum element is at most the average, we obtain that with probability at least
  $1-\delta$ it holds
  $$
  \min_{i\in [T]} \snorm{2}{\nabla f(\vec w^{(i)})} \leq \eps \;.
  $$
 This completes the proof.
\end{proof}

\subsection{Proof of Lemma~\ref{lem:sigmoid_smoothness}} \label{app:smooth-lemma}

We start with the following more general lemma from which 
we can deduce Lemma~\ref{lem:sigmoid_smoothness}.

\begin{lemma}[Objective Properties] \label{lem:high_dimensional_objective_properties}
Let $\D$ be a distribution on $\R^d \times \{-1, +1\}$ such that the marginal
$\D_{\bx}$ on $\R^d$ is in isotropic position. 
Let $g(\vec x, y,\vec w) = f(-y \dotp{{\vec w}}{\vec x})$ and 
\[ \SL({\vec w}) = \E_{(\vec x, y) \sim \D} \lp[g( \vec x, y,\vec w) \rp] \;.  \]
Assume that $f$ is a twice differentiable function on $\R$ such that 
$|f(t)| \leq R$, $|f'(t)| \leq B$, and $f''(t) \leq  K$ for all $t \in \R$.
Then $\SL({\vec w})$ is continuously differentiable, 
$|\SL({\vec w})| \leq R$ for all ${\vec w}$ in $\W = \{{\vec w} : \snorm{2}{{\vec w}}\geq 1 \}$,
$\E_{(\bx, y) \sim \D}[\snorm{2}{\nabla_{{\vec w}} g( \vec x, y,\vec w)}^2] \leq 4 B^2 d$,
$\snorm{2}{\E_{(\bx, y) \sim \D}[\nabla_{{\vec w}} g( \vec x, y,\vec w)]}^2 \leq 3 B^2$,
and $\nabla_{{\vec w}} \SL({\vec w})$ is $(6B + 4K)$-Lipschitz.
\end{lemma}
\begin{proof}
Write $g( \vec x, y,\vec w) = f(\ell ({\vec w}, \vec x) y )$, where
$\ell({\vec w}, \vec x) = \dotp{{\vec w}}{\vec x}/\snorm{2}{{\vec w}}$.
Note that $|g(\vec x, y,\vec w)| \leq R$. Therefore, $|\SL({\vec w})| \leq R$.

We now deal with the function 
$\ell({\vec w}, \vec x) = \dotp{{\vec w}}{\vec x}/\snorm{2}{\vec w}$.  
We have that 
$\nabla_{{\vec w}} \ell({\vec w}, \vec x) = \frac{\vec x}{\snorm{2}{{\vec w}}} - \dotp{{\vec w}}{\vec x} \frac{{\vec w}}{\snorm{2}{{\vec w}}^3}$. 
Observe that
$\snorm{2}{\nabla_{{\vec w}} \ell({\vec w},\vec x)} \leq 2 \snorm{2}{\vec x}/\snorm{2}{\vec w} \leq 2 \snorm{2}{\vec x}$.
Therefore, since $\D_{\bx}$ is isotropic, we get that 
$\E_{(\bx, y) \sim \D}[\snorm{2}{\nabla_{{\vec w}} g( \vec x, y,\vec w)}^2] \leq 4 B^2 \E_{(\bx, y) \sim \D}[\snorm{2}{\vec x}^2] = 4 B^2 d$.
Moreover, we have
\begin{align*}
\snorm{2}{\E_{(\bx, y) \sim \D}[\nabla_{{\vec w}} g( \vec x, y,\vec w)]}^2 
&=\left(\sup_{\snorm{2}{\vec v}=1} \E_{(\bx, y) \sim \D}[\dotp{\nabla_{{\vec w}} g( \vec x, y,\vec w)}{\vec v}]\right)^2\\
&\leq B^2 \left(\sup_{\snorm{2}{\vec v}=1} \E_{\bx\sim \D_{\bx}}[\dotp{\nabla_{{\vec w}} \ell({\vec w}, \vec x)}{\vec v}]\right)^2 \\	
&\leq B^2 \left(\sup_{\snorm{2}{\vec v}=1} \E_{\bx\sim \D_{\bx}}\left[\frac{|\dotp{\vec x}{ \vec v}|}{\snorm{2}{{\vec w}}} +
|\dotp{{\vec w}}{\vec x}| \frac{|\dotp{\vec w}{ \vec v}|}{\snorm{2}{{\vec w}}^3}\right]\right)^2 \\
&\leq  B^2 \left(2\sup_{\snorm{2}{\vec v}=1} \sqrt{\E_{\bx\sim \D_{\bx}}\left[|\dotp{\vec x}{ \vec v}|^2\right]}\right)^2
\leq 4B^2\;,
\end{align*}
where in the first inequality we used $f'(t)\leq B$ and in the third we used the Cauchy-Swartz inequality 
and that $\snorm{2}{\vec w}\geq 1$.

We finally prove that the gradient of $\SL$ is Lipschitz.  We have that
\[
\nabla^2_{{\vec w}} \ell({\vec w}, \vec x) = - \frac{\vec x {\vec w}^T}{\snorm{2}{{\vec w}}^3} 
- \frac{{\vec w} \vec x^T}{\snorm{2}{{\vec w}}^3} 
- \frac{ \dotp{\vec x}{ {\vec w}}}{\snorm{2}{{\vec w}}^3} \vec I
+ 3 \dotp{\vec x}{ {\vec w}} \frac{{\vec w}{\vec w}^T}{\snorm{2}{{\vec w}}^5} \;.
\]
Therefore,
  \begin{align*}
    \nabla_{{\vec w}}^2 g(\vec x, y,{\vec w})
    &=
    f''(y \ell({\vec w}, \vec x)) \nabla_{{\vec w}} \ell({\vec w}, \vec x)
    \nabla_{{\vec w}} \ell({\vec w}, \vec x)^T
    + f'(\ell({\vec w}, \vec x))
    \nabla_{{\vec w}}^2 \ell({\vec w}, \vec x)
    \\
    &=
    f''(y \ell({\vec w}, \vec x))
    \left(\frac{\vec x \vec x^T}{\snorm{2}{{\vec w}}^2}
      - \frac{\dotp{{\vec w}}{ \vec x}}{\snorm{2}{{\vec w}}^4} {\vec w} \vec x^T
      - \frac{\dotp{{\vec w}}{ \vec x}}{\snorm{2}{{\vec w}}^4} \vec x {\vec w}^T
      + \frac{\dotp{{\vec w}}{ \vec x}^2}{\snorm{2}{{\vec w}}^6} {\vec w} {\vec w}^T
    \right)
    \\
    &+
    f'(y \ell({\vec w}, x)) y\ \nabla^2_{{\vec w}} \ell({\vec w}, \vec x).
  \end{align*}
  To prove that $\SL({\vec w})$ has Lipschitz gradient, we will bound
  $\snorm{2}{\nabla_{{\vec w}}^2 \SL(\vec w)}$.
  Let ${\vec v} \in \mathbb{S}^{d-1}$. We have
  \begin{align*}
    \left| \dotp{{\vec v}} {\E_{(\bx, y) \sim \D}\left[
      \frac{f''(y \ell({\vec w}, \vec x))}{\snorm{2}{{\vec w}}^2} \vec x \vec x^T\right]
    {\vec v}}
    \right|
    \leq
    \E_{(\bx, y) \sim \D}\left[\frac{
    |f''(y \ell({\vec w}, \vec x))|}{\snorm{2}{{\vec w}}^2} \dotp{\vec x} {{\vec v}}^2\right]
    \\
    \leq
    \frac{K}{\snorm{2}{{\vec w}}^2}
    \E_{(\bx, y) \sim \D}\left[\dotp{\vec x} {{\vec v}}^2\right]
    \leq \frac{K}{\snorm{2}{{\vec w}}^2} \;,
  \end{align*}
where we used the fact that $|f''(t)| \leq K$ for all $t$.
To get the last equality, we used the fact that the marginal distribution on $\vec x$ is isotropic.
Similarly, we have
  \begin{align*}
    \left| \dotp{{\vec v}} {\E_{(\bx, y) \sim \D}
      \left[
      \frac{f''(y \ell({\vec w}, \vec x))}{\snorm{2}{{\vec w}}^4} \dotp{{\vec w}}{ \vec x} {\vec w} \vec x^T}
    \right]
    {\vec v}
    \right|
    \leq
    \E_{(\bx, y) \sim \D}
    \left[
      \frac{|f''(y \ell({\vec w}, \vec x))|}{\snorm{2}{{\vec w}}^4}
      |\dotp{{\vec w}} {\vec x}| |\dotp{{\vec v}}{ {\vec w}}|
      |\dotp{\vec x}{{\vec v}}|
    \right]
    \\
    \leq
    \frac{K}{\snorm{2}{{\vec w}}^3}
    \E_{(\bx, y) \sim \D}
    \left[
      |\dotp{{\vec w}}{ \vec x}| |\dotp{\vec x}{\vec v}|
    \right]
    \leq
    \frac{K}{\snorm{2}{{\vec w}}^3}
    \sqrt{ \E_{(\bx, y) \sim \D} \left[\dotp{\vec w}{ \vec x}^2\right] }
    \sqrt{ \E_{(\bx, y) \sim \D} \left[\dotp{\vec x}{\vec v}^2\right] }
    \leq
    \frac{K}{\snorm{2}{{\vec w}}^3} \;,
  \end{align*}
where the last step follows because the distribution $\D_{\bx}$ is isotropic. 
Similarly, we can bound the rest of the terms of $|{\vec v}^T \nabla_{{\vec w}}^2 \SL({\vec w}) {\vec v}|$ 
to obtain
\[
    |{\vec v}^T \nabla_{{\vec w}}^2 \SL({\vec w}) {\vec v}| \leq
    B
    \lp( \frac{2}{\snorm{2}{{\vec w}}^2}
    + \frac{4}{\snorm{2}{{\vec w}}^3}
    \rp)
    + K
    \lp( \frac{1}{\snorm{2}{{\vec w}}^2}
    + \frac{2}{\snorm{2}{{\vec w}}^3}
    + \frac{1}{\snorm{2}{{\vec w}}^4}
    \rp)
    \leq 6 B + 4 K \;,
\]
  where we used the fact that $\snorm{2}{{\vec w}} \geq 1$.
\end{proof}

Our desired lemma now follows as a corollary.

\begin{customlem}{\ref{lem:sigmoid_smoothness}}[Sigmoid Smoothness]
  \textit{
  Let $S_\sigma(t) = 1/(1 + e^{-t/\sigma})$ and $\SL(\bw) = \E_{(\vec x,y) \sim
    \D} \left[S_{\sigma}\left(-y \frac{\dotp{\vec w}{\vec x}}{\snorm{2}{\vec w}}
    \right)\right]$, for $\bw \in \W$, where $\W=\{\bw\in \R^d:
  \snorm{2}{\bw}\geq 1\}$.  We have that $\SL(\vec w)$ is continuously
  differentiable in $\W$, $|\SL(\vec w)| \leq 1$, $\E_{(\bx, y) \sim
  \D}[\snorm{2}{\nabla_{{\vec w}} S_{\sigma}({\vec w}, \vec x, y)}^2] \leq
  4d/\sigma^2$, $\snorm{2}{\nabla_{\vec w} \SL(\vec w)}^2 \leq 4/\sigma^2$, and
  $\nabla_{\vec w} \SL(\vec w)$ is $(6/\sigma + 12 /\sigma^2)$-Lipschitz.
}
\end{customlem}

\begin{proof}
  We first observe that $|S_{\sigma}(t)| \leq 1$ for all $t$ in $\R$.  Moreover, $S_{\sigma}$
  is continuously differentiable. The first and the second derivative of $S_{\sigma}$ with respect to $t$ is
  $$
  S'_{\sigma}(t) = S_{\sigma}^2(t) \frac{e^{-t/\sigma}}{\sigma}
  \quad \text{and}\quad
  S''_{\sigma}(t) = S_{\sigma}^3(t) \frac{2 e^{-2 t/\sigma}}{\sigma^2}
  - S_{\sigma}^2(t) \frac{e^{-t/\sigma}}{\sigma^2} \;.
  $$
  We have that $S'_{\sigma}(t) \leq S'_{\sigma}(0) = 1/\sigma$
  and $S''_{\sigma}(t) \leq 3/\sigma^2$.
  The result follows by applying Lemma~\ref{lem:high_dimensional_objective_properties}.
\end{proof}

\appendix

\end{document}